\RequirePackage[OT1]{fontenc}
\documentclass[10 pt, Journal]{IEEETran}
\usepackage{amssymb, amsfonts, mathtools, xargs, tensor, units, cite, stmaryrd, mathrsfs, algpseudocode, graphicx, comment, amsmath, amsfonts,amssymb, listings, algorithm, tikz,pgfplots,subcaption,amssymb, amsfonts, mathtools, xargs, tensor, soul, units, cite, stmaryrd, mathrsfs, algpseudocode, algorithm, graphicx, color,amsthm,blindtext,pgfplots,multirow,booktabs,textcomp,stfloats,url,xcolor,cancel,enumitem}

\usepackage[normalem]{ulem} 

\pgfplotsset{compat=1.18}
\usetikzlibrary{3d, calc}
\pgfplotsset{compat=newest}
\newtheorem{assumption}{Assumption}

\newtheorem{remark}{Remark}
\newtheorem{lemma}{Lemma}
\newtheorem{theorem}{Theorem}
\newtheorem{corollary}{Corollary}

\DeclareMathOperator{\Int}{Int}

\definecolor{CG}{rgb}{0.1,0.5,0.2}

\definecolor{custom_green}{rgb}{0.1,0.8,0.4}

\begin{document}
\title{\LARGE \bf A Hough transform approach to safety-aware scalar field mapping using Gaussian Processes}
\author{Muzaffar Qureshi$^{1}$ \and Trivikram Satharasi$^{1}$  \and Tochukwu E. Ogri$^{1}$ \and Kyle Volle$^{2}$ \and Rushikesh Kamalapurkar$^{1}$
\thanks{This research was supported in part by the Air Force Research Laboratory under contract number FA8651-24-1-0019. Any opinions, findings, or recommendations in this article are those of the author(s), and do not necessarily reflect the views of the sponsoring agencies.}%
\thanks{$^{1}$ Department of Mechanical and Aerospace Engineering, University of Florida, Gainesville, FL, USA, email: {\tt\footnotesize \{muzaffar.qureshi, t.satharasi, tochukwu.ogri, rkamalapurkar\}@ufl.edu}.}%
\thanks{$^{2}$ Torch Technologies, Shalimar, Florida, USA, email: {
\tt \footnotesize Kyle.Volle@torchtechnologies.com}}}



\maketitle

\begin{abstract}
This paper presents a framework for mapping unknown scalar fields using a sensor-equipped autonomous robot operating in unsafe environments. The unsafe regions are defined as regions of high-intensity, where the field value exceeds a predefined safety threshold. For safe and efficient mapping of the scalar field, the sensor-equipped robot must avoid high-intensity regions during the measurement process. In this paper, the scalar field is modeled as a sample from a Gaussian process (GP), which enables Bayesian inference and provides closed-form expressions for both the predictive mean and the uncertainty. Concurrently, the spatial structure of the high-intensity regions is estimated in real-time using the Hough transform (HT), leveraging the evolving GP posterior. A safe sampling strategy is then employed to guide the robot towards safe measurement locations, using probabilistic safety guarantees on the evolving GP posterior. The estimated high-intensity regions also facilitate the design of safe motion plans for the robot. The effectiveness of the approach is verified through two numerical simulation studies and an indoor experiment for mapping a light-intensity field using a wheeled mobile robot.
\end{abstract}


\section{Introduction}
\IEEEPARstart{M}{apping} the spatial and temporal variations of an unknown scalar field remains an active area of research, having broad applications in environmental monitoring, ocean floor mapping, radiation detection, fire hazard localization, and wireless signal estimation \cite{muzaffar.Sabella1988,SCC.Lin.Liu.ea2019,SCC.Matsuda.Nozaki.ea2023,SCC.Ren.Li.ea2021}. To construct an accurate map of the scalar field, a common approach is to deploy a sensor-equipped robot to visit predefined locations and collect field measurements \cite{SCC.Razak.Sukumar.ea2019,SCC.La.Sheng.ea2015,SCC.Dossing.Silva.ea2021}. Mobile sensor networks also provide an efficient and practical solution for these mapping tasks, enabling comprehensive data acquisition over large domains using multiple sensor-equipped robots \cite{muzaffar.La.Sheng2013,muzaffar.Lin.AlAbri.ea2020,muzaffar.Nguyen.La.ea2015}. 

In this paper, we focus on mapping a scalar field that can have detrimental effects on the sensor-carrying robot. For example, in applications that require mapping nuclear radiation, forest fires, chemical spills, or hostile radar signals, the robot needs to maintain a safe distance from regions of high-intensity to avoid damage or detection. \cite{muzaffar.Groves.Hernandez.ea2021, muzaffar.Aleotti.Micconi.ea2017, muzaffar.Ghamry.Zhang2016, muzaffar.Casbeer.Li.ea2015}. 

Since the field is unknown, the robot must leverage real-time measurements to infer and avoid the high-intensity regions. Avoidance of high-intensity regions necessitates a framework that concurrently generates an accurate estimate of the scalar field and identifies high-intensity regions in the domain. Safe motion plans can then be generated for the robot, leveraging the estimated high-intensity regions.

Existing approaches to scalar field estimation offer robust techniques for reconstructing spatial distributions from sparse measurements. These approaches emphasize large-scale coverage and optimization of resource usage, such as minimizing control effort for mobile platforms like robots and UAVs \cite{muzaffar.Sabella1988,SCC.Lin.Liu.ea2019,SCC.Matsuda.Nozaki.ea2023,SCC.Ren.Li.ea2021,SCC.Razak.Sukumar.ea2019,SCC.La.Sheng.ea2015,SCC.Dossing.Silva.ea2021, muzaffar.La.Sheng2013,muzaffar.Lin.AlAbri.ea2020,muzaffar.Nguyen.La.ea2015}. Estimation techniques such as \cite{muzaffar.Jayasekaramudeli.Leong.ea2024,muzaffar.Nguyen.La.ea2015} rely on static sensor networks or pre-defined robot trajectories to provide resource-efficient field mapping. These methods aim to optimize the mapping efficiency but do not consider risks related to exposure to high field intensity. 

On the other hand, extremum seeking (ES) algorithms \cite{muzaffar.Krstic.Wang2000, muzaffar.Ariyur.Krstic2003, muzaffar.Ghaffari.Krstic.ea2012} are well-regarded for their efficiency in locating extrema within scalar fields, employing exploitative strategies that guide robots toward local or global optima. Recent advancements have introduced safety-aware mechanisms into the ES framework; notably, \cite{muzaffar.Williams.Krstic.ea2022, muzaffar.Williams.Krstic.ea2025} utilize control barrier function-based filters to mitigate safety violations during transient dynamics. However, these algorithms remain primarily focused on extremum localization, often at the expense of broader exploration and comprehensive field mapping. Consequently, they do not address the objective of constructing a full spatial estimate of the scalar field.

Obstacle avoidance techniques \cite{muzaffar.Katona.Neamah.ea2024,muzaffar.Kunchev.Jain.ea2006,muzaffar.Minguez.Lamiraux.ea2008} are often employed for safe navigation while mapping given scalar fields. Obstacles within the operating environment can be detected with sensors such as LiDAR, Sonar, or depth cameras \cite{muzaffar.BiancaSangiovanni, muzaffar.Li.Wu2022, DrMao}. These sensors provide real-time data within a finite sensing radius, enabling the robot to react to obstacles before violating safety. However, in this paper, \emph{obstacles} are not physical, they are regions where the intensity of the scalar field surpasses a safety threshold. As such, the methods introduced in \cite{muzaffar.Katona.Neamah.ea2024,muzaffar.Kunchev.Jain.ea2006,muzaffar.Minguez.Lamiraux.ea2008,muzaffar.BiancaSangiovanni, muzaffar.Li.Wu2022} do not apply to the safe scalar field mapping problem considered in this paper. Since the sensors only measure the field intensity at the current position, the robot must predict intensities in surrounding areas to map and avoid high-intensity regions.

In this paper, we focus on Gaussian process (GP) regression for scalar field mapping \cite{SCC.Rasmussen.Williams2006}. A key advantage of GP is the computation of a predictive covariance that naturally serves as a utility function for informative exploration \cite{muzaffar.Binney.Krause.ea2010, muzaffar.Hitz.Galceran.ea2017}. For example, the trajectories for the robot can be planned to visit locations with high predictive covariance. While effective, standard GP-based planners do not explicitly incorporate safety constraints imposed by potentially hazardous high-intensity regions.

To ensure safe navigation during the mapping of unknown scalar fields, this paper integrates GP regression with the Hough transform (HT), a robust feature extraction technique widely used in image processing \cite{muzaffar.Yuen.Princen.ea1990, muzaffar.Duda.Hart1972}. The authors presented preliminary results on detection of high-intensity regions in scalar fields in \cite{muzaffar.Qureshi.Ogri.ea2024}. However, \cite{muzaffar.Qureshi.Ogri.ea2024} did not incorporate an optimal strategy for selection of the initial measurement locations, limiting the overall efficiency of the mapping process. In particular, for a fixed number of measurements, selecting uninformative locations can result in degraded mapping performance and increased posterior uncertainty across the domain \cite{SCC.Rasmussen.Williams2006}.

In this study, we address the aforementioned limitation by utilizing optimal sampling strategies, such as \emph{maximum variance sampling (MVS)} \cite{muzaffar.Chaloner.Verdinelli1995}, to optimize the number of measurements
required for mapping. Additionally, the information loss is quantified and compared between scenarios where measurement locations are restricted to safe regions versus when the entire domain is available for sampling using the MVS strategy. The framework developed in this work is also extended to integrate path planning and obstacle avoidance using algorithms such as RRT* \cite{muzaffar.Karaman.Frazzoli2011}. Finally, the performance of the developed GP-HT algorithm is validated through two distinct numerical simulations, as well as an indoor scalar field mapping experiment using a sensor-equipped wheeled mobile robot.

A central challenge addressed in this work is balancing safety with information gain during the mapping of scalar fields. This trade-off becomes particularly complex when informative locations lie within unsafe zones. Consequently, a fundamental problem arises in designing a budget-aware relocation scheme that optimally replaces unsafe measurement candidates. Specifically, given a fixed measurement budget and HT-derived safety estimates, the planner must identify admissible safe points that balance predictive covariance reduction and pathwise cost. This paper addresses this challenge by formulating a systematic relocation algorithm, enabling a rigorous assessment of the information loss incurred by safety-constrained sampling compared to unconstrained MVS.

The remainder of this paper is structured as follows. Section~\ref{sec: problem_formulation} presents the problem formulation. Sections~\ref{sec:GP_section}- IV contain the details of the developed methodology, including the GP regression model and the HT-based detection algorithm. An information loss analysis along with the convergence analysis is presented in Section V and ~\ref{sec:convergence_analysis}, respectively. Section~\ref{sec:simulations} provides numerical simulations and experimental results. A discussion of the findings and the summary of contributions and ideas for future research is offered in Section~\ref{sec:discussion}.

\subsection{Notation}
In this paper, $\mathbb{R}$ denotes the set of real numbers, $\mathbb{R}^n$ and $\mathbb{R}^{n \times m}$ denote the sets of real $n$-dimensional vectors and $n \times m$ dimensional matrices, respectively. $\mathbb{R}_{> 0}$ denotes the set of positive real numbers, and $\mathbb{R}_{\geq a}$ denotes the set of real numbers greater than or
equal to $a \in \mathbb{R}$. The symbol $\| \cdot \|$ represents the $2-$norm for vectors and the induced $2-$norm for matrices. $\mathbf{I}_{n}$ represents an $n \times n$ identity matrix and $\mathbf{0}_n$ represents an $n \times n$ zero matrix. For any matrix $A \in \mathbb{R}^{n \times n}$, $\lambda_{\max}(A)$ and $\lambda_{\min}(A)$ denote the maximum and minimum eigenvalues of matrix $A$, respectively. For a vector $\mathbf{x}$ or a matrix $\mathbf{A}$, its transpose is denoted by $\mathbf{x}^\top$ and $\mathbf{A}^\top$, respectively. The gradient of a scalar-valued function $ f: \mathbb{R}^n \rightarrow \mathbb{R} $ is denoted as $ \nabla f(x) = \left[ \frac{\partial f}{\partial x_1}, \dots, \frac{\partial f}{\partial x_n} \right]^\top \in \mathbb{R}^n $, where $x \in \mathbb{R}^n$. For a given set $ \mathcal{X}$, $ \partial \mathcal{X} $ denotes the boundary, $ \Int(\mathcal{X}) $ denotes the interior, and \( |\mathcal{X}| \) denotes the cardinality of $\mathcal{X}$.

\section{Problem Formulation}\label{sec: problem_formulation}
Consider an unknown scalar field \( f : \mathcal{X} \rightarrow \mathbb{R} \), where \( \mathcal{X} \subset \mathbb{R}^n \) is a compact domain and \( n \in \{2, 3\} \) denotes the spatial dimension. Given a safety threshold $\bar{f} \in \mathbb{R}$, the domain \( \mathcal{X} \) is partitioned into two disjoint regions: the safe region denoted as \( \mathcal{X}^s \) and the high-intensity or unsafe region denoted as \( \mathcal{X}^u \), defined as
\begin{align}
\mathcal{X}^s &\coloneqq \left\{ x \in \mathcal{X} \;\middle|\; f(x) \leq \bar{f} \right\}, \label{eq:safe_region} \\
\mathcal{X}^u &\coloneqq \left\{ x \in \mathcal{X} \;\middle|\; f(x) > \bar{f} \right\}, \label{eq:high_intensity_region}
\end{align}
respectively, where \( \bar{f} \in  (f_{\min},  f_{\max}) \), 
$f_{\min} \coloneqq \min_{x \in \mathcal{X}} f(x)$,  and
$f_{\max} \coloneqq \max_{x \in \mathcal{X}} f(x)$.

In this development, the unsafe region $ \mathcal{X} ^u$ is assumed to be a finite union of disjoint $P$ sets, where each subset corresponds to a distinct field source in $\mathcal{X}$, such that
\begin{equation}
\mathcal{X}^u = \bigcup_{i=1}^{P} \mathcal{X}^{u,i}.
\label{eq:unsafe_region_union}
\end{equation}
The number of high-intensity regions and their spatial profiles are unknown \emph{a priori} and must be inferred online during the measurement process by the robot\footnote{The term spatial profile refers to the geometric characteristics of each super-level set, such as its shape and extent in $\mathcal{X}$.}. The following assumptions are required to facilitate the analysis in subsequent sections.
\begin{assumption}\cite{muzaffar.Schoelkopf.Smola2001}\label{ass:boundedF}
The domain $\mathcal{X}$ is endowed with a positive definite kernel function $k : \mathcal{X} \times \mathcal{X} \rightarrow \mathbb{R}$ and \( f\) has a bounded norm in the native reproducing kernel Hilbert space (RKHS) $\mathcal{H}_k$ of $k$, satisfying $\|f\|_{\mathcal{H}_k} \le B$ for some $B > 0$.
\end{assumption}
\begin{assumption}\label{ass:Lip_bound} \( f : \mathcal{X} \to \mathbb{R} \) is Lipschitz continuous on $\mathcal{X}$, i.e., there exists a constant \( L_f > 0 \) such that $\vert f({x}) - f({x}') \vert \le L_f \|{x} - {x}'\|_2, \, \forall\, {x}, {x}' \in \mathcal{X}$.

\end{assumption}\noindent
To facilitate tractable inference over the domain $\mathcal{X}$, a GP is employed as a probabilistic model for the unknown field $f$\begin{equation}
f\sim \mathcal{GP}(\mu, \sigma),
\label{eq:GP_model}
\end{equation}
where \( \mu: \mathbb{R}^n \rightarrow \mathbb{R} \) is the mean function and \( \sigma: \mathbb{R}^n \times \mathbb{R}^n \rightarrow \mathbb{R} \) is the covariance function. It is noted that while $f$ is modeled via a GP for estimation, the theoretical analysis treats $f$ as a function residing in the RKHS $\mathcal{H}_k$, consistent with \cite{muzaffar.Srinivas.Krause.ea2012}.

To map \( f \), a finite set of measurement locations in \( \mathcal{X} \) is visited to record the field intensity. The measurement collected at position $x_i$ is denoted as $y_i$ and is assumed to be of the form
\begin{equation}
y_i = f(x_i) + \varepsilon_i.
\label{eq:measurement_model}
\end{equation}
where \( \varepsilon_i \sim\mathcal{N}(0, \sigma_n^2) \) is zero-mean Gaussian noise with known variance \( \sigma_n^2 \). The mapping objectives considered in this paper are
\begin{itemize}
    \item \textit{Field Mapping:} Construct an estimate of \( f \) using a finite set of noisy measurements such that the error between the true and estimated field is minimized in safe regions.
    \item \textit{High Intensity Region Avoidance:}  Optimize the measurement locations to minimize the number of measurements from the unknown unsafe set \( \mathcal{X}^u \).
    \item \textit{Safe Trajectory Generation:} Given a sequence of measurement locations, generate collision-free trajectories that ensure the robot avoids the unsafe region \( \mathcal{X}^u \) while navigating between any two measurement locations.
\end{itemize}
The aforementioned tasks are motivated by real-world robotics applications, such as mapping fire fronts in forest monitoring \cite{muzaffar.Lin.Liu.ea2019,muzaffar.Hossain.Zhang.ea2019}, avoiding hostile radar emitters in reconnaissance missions \cite{muzaffar.Pack.DeLima.ea2009}, and characterizing radioactive contamination in nuclear disaster zones \cite{muzaffar.Huo.Liu.ea2020,muzaffar.Ardiny.Witwicki.ea2019,muzaffar.Qian.Song.ea2012}.

\section{Gaussian Process Regression}\label{sec:GP_section}
This section outlines a GP regression framework for mapping the unknown function \( f \) based on the noisy measurements collected by a sensor-equipped autonomous robot. GP regression offers a flexible, non-parametric Bayesian approach for learning complex and nonlinear functions \cite{SCC.Rasmussen.Williams2006}. An important component of GP regression framework is the selection of the kernel function, which encodes prior assumptions about the structure and smoothness of the function being mapped. In this paper, the squared exponential (SE) covariance kernel is selected for GP regression, which is defined as
\begin{equation}
k (x, x') \coloneqq \alpha^2 \exp \left( -\frac{\|x - x'\|^2}{2 l^2} \right),
\label{eq:kernel_function}
\end{equation}
where \( \alpha \in \mathbb{R}_{>0} \) denotes the signal variance and \( l \in \mathbb{R}_{>0} \) denote the length-scale of the kernel function. 

Let \( \mathcal{X}_t= \{ {x}_i \}_{i=1}^{t} \) denote the set of measurement locations,  \( \mathcal{Y}_t = \{ {y}_i \}_{i=1}^{t} \) denote the corresponding field measurements, and \( \mathcal{D}_t = \{ ({x}_i, {y}_i) \}_{i=1}^{t} \) denote the dataset at the measurement step \( t \). The posterior distribution of the GP at any given measurement step \( t \), conditioned on the data \( \mathcal{D}_t \), is expressed as
\begin{equation}
f(x^*) \mid \mathcal{D}_t \sim \mathcal{N}\left( \mu_t(x^*), \sigma_t^2(x^*) \right),
\end{equation}
where $x^*$ denotes an arbitrary test point, \( \mu_t  \) is the posterior mean and \( \sigma_t^2 \) is the posterior variance of the GP. 

Given the measurements \( \mathcal{Y}_t \) at \( \mathcal{X}_t \), an arbitrary test location $x^*$, and the kernel function in \eqref{eq:kernel_function}, the kernel vector $\mathbf{k}_t(x^*) \in \mathbb{R}^t$ between the test point and the measurement locations, and the kernel matrix $K_{\mathrm{t}\mathrm{t}}\in \mathbb{R}^{t \times t}$ of the measurement locations can be defined as \cite{SCC.Rasmussen.Williams2006}
\begin{equation}
\mathbf{k}_\mathrm{t}(x^*) \coloneqq \left[ k(x^*, x_1), \dots, k(x^*, x_\mathrm{t}) \right]^\top,
\label{eq:k_vector}
\end{equation}
\begin{equation}
[K_{\mathrm{t}\mathrm{t}}]_{i,j} \coloneqq k({x}_i, {x}_j), \quad i,j = 1,\dots,\mathrm{t}.
\label{eq:Ktt}
\end{equation}
Using $\mathbf{k}_\mathrm{t}$ and $K_{\mathrm{t}\mathrm{t}}$, the GP posterior mean \( \mu_t \) and variance \( \sigma^2_t \) at any measurement step $t$ are computed as \cite{SCC.Rasmussen.Williams2006}
\begin{equation}
\mu_t(x^*) = \mathbf{k}_\mathrm{t}(x^*)^\top \left( K_{\mathrm{t}\mathrm{t}} + \sigma_n^2 \mathbf{I}_t \right)^{-1} \mathcal{Y}_t,
\label{eq:mean_eq}
\end{equation}
\begin{equation}
\sigma^2_t(x^*) = k(x^*, x^*) - \mathbf{k}_\mathrm{t}(x^*)^\top \left( K_{\mathrm{t}\mathrm{t}} + \sigma_n^2 \mathbf{I}_t \right)^{-1} \mathbf{k}_\mathrm{t}(x^*).
\label{eq:var_eq}
\end{equation}Building on the GP regression framework, the next subsection focuses on information-driven initial selection of planned measurement locations. This initial selection is then modified in Section \ref{sec:hough_transform_estimation} to minimize exposure to high-intensity regions.

\subsection{Maximum Covariance Sampling Strategy}
Existing strategies for the selection of measurement locations aim to reduce the predictive variance by maximizing the information gained from each measurement. A widely used approach is mutual information sampling, which selects the next measurement location expected to yield the largest reduction in predictive variance~\cite{muzaffar.Cover.Thomas2005,muzaffar.Chaloner.Verdinelli1995}.

The mutual information between the true function values \( F_{\mathcal{X}_t} \coloneqq [f(x_1), \ldots, f(x_t)]^\top \) and the noisy measurements \( \mathcal{Y}_t \) is defined as (cf. \cite[Chapter~2]{muzaffar.Cover.Thomas2005})
\begin{equation}\label{eq:information_gain}
I(\mathcal{Y}_t; F_{\mathcal{X}_t}) \coloneqq H(\mathcal{Y}_t) - H(\mathcal{Y}_t \mid F_{\mathcal{X}_t}),
\end{equation}
where \( H(\mathcal{Y}_t) = -\sum_{\substack{y \in \mathcal{Y}_t}} p(y) \ln p(y) \) is the entropy of \( \mathcal{Y}_t \) and \( H(\mathcal{Y}_t\mid F_{\mathcal{X}_t}) = - \sum_{\substack{z} \in F_{\mathcal{X}_t}} \sum_{\substack{y \in \mathcal{Y}_t}} p(z,y) \ln p(y \mid z) \) is the conditional entropy of \( \mathcal{Y}_t \) given \( F_{\mathcal{X}_t} \). For Gaussian distributions, the mutual information takes the closed-form expression \cite{muzaffar.Cover.Thomas2005}
\begin{equation}\label{eq:cov_matrix_eqn}
I(\mathcal{Y}_t; F_{\mathcal{X}_t}) = \frac{1}{2} \ln\left( \det \left( \mathbf{I}_t + \frac{1}{\sigma_n^2} K_{\mathrm{t}\mathrm{t}} \right)\right),
\end{equation}
where \( K_{\mathrm{t}\mathrm{t}} \in \mathbb{R}^{t \times t} \) is the kernel matrix defined in \eqref{eq:Ktt}.

Finding an optimal set of measurement locations denoted as $\mathcal{X}_T^o$ that maximizes $I$ over $\mathcal{X}$ is NP-hard\cite{muzaffar.Ko.Lee.ea1995}. A commonly employed approach to solve this problem is to use greedy sampling as developed in \cite{muzaffar.Krause.Guestrin2012}, where at each step $t$, the next sampling point is selected to maximize the information gain using the rule
\begingroup\medmuskip=0mu\begin{equation}\label{eq:mutual_information_update}
x_{t+1} = \underset{x \in \mathcal{X} \setminus \mathcal{X}_t}{\operatorname{argmax}} \left[ I(\mathcal{Y}_{\mathcal{X}_t \cup \{ x \}}; F_{\mathcal{X}_t \cup \{ x \}}) - I(\mathcal{Y}_{\mathcal{X}_t}; F_{\mathcal{X}_t}) \right].
\end{equation}
\endgroup
Let the optimal information gain over any \( T \) locations be defined as
\begin{equation}
\label{eq:gamma_t_optimal}
\gamma_T^o \coloneqq \max_{\substack{\mathcal{X}_T^o \subseteq \mathcal{X} \\ |\mathcal{X}_T^o| = T}} I(\mathcal{Y}_{\mathcal{X}_T^o}; F_{\mathcal{X}_T^o}),
\end{equation}
and let 
\begin{equation}
\label{eq:gamma_t_greedy}
\gamma_T^g \coloneqq I(\mathcal{Y}_{\mathcal{X}_T^g}; F_{\mathcal{X}_T^g}),
\end{equation}
denote the information gain using the greedy sampling algorithm defined in \eqref{eq:mutual_information_update}, and \( \mathcal{X}_T^g \) is the set of measurement locations selected by the greedy algorithm after \( T \) iterations. Using the submodular property of information function $I$ in \eqref{eq:information_gain}, the greedy algorithm yields a near-optimal solution that satisfies the inequality (cf.~\cite{muzaffar.Nemhauser.Wolsey.ea1978})
\begin{equation}
\label{eq:submodular_bound}
\gamma_T^g \geq \left( 1 - \frac{1}{e} \right) \gamma_T^o.
\end{equation}
The factor \( \left( 1 - \frac{1}{e} \right) \approx 0.632 \) in~\eqref{eq:submodular_bound} guarantees that the information gain from the greedy strategy is at least 63.2\% of that obtained by the globally optimal set $\mathcal{X}^o_T$ (see~\cite{muzaffar.Nemhauser.Wolsey.ea1978, muzaffar.Krause.Guestrin2012}). The information gain serves as a quantitative metric to compare the learning performance of different sampling strategies. Specifically, the bound established in \eqref{eq:submodular_bound} provides a baseline for assessing the information loss incurred when deviating from the greedy strategy as discussed in detail in Section~\ref{sec:information_loss_Gp_HT}.

In the Gaussian setting, where measurements have Gaussian noise, the greedy maximization of information gain simplifies to the selection of the point with the highest variance~\cite{muzaffar.Ko.Lee.ea1995,muzaffar.Krause.Singh.ea2008}. Specifically, the next measurement location is determined by
\begin{equation}\label{eq:var_criterion}
x_{t+1} = \underset{x \in \mathcal{X} \setminus \mathcal{X}_t}{\operatorname{argmax}}\ \sigma^2_t(x),
\end{equation}
where $\sigma^2_t(x)$ is the posterior variance defined in \eqref{eq:var_eq}. This strategy is commonly referred to as maximum variance sampling (MVS) \cite{muzaffar.Chaloner.Verdinelli1995}. Since the covariance computation in~\eqref{eq:var_eq} does not require field measurements, the selection criterion in~\eqref{eq:var_criterion} can be solved offline to yield the most informative set of planned measurement locations. 

 In this paper, it is assumed that an initial set of planned measurement locations, denoted as $\mathcal{X}_T^0$, is pre-computed using MVS. The order in which the measurement locations appear in $\mathcal{X}_T^0$ is dictated by \eqref{eq:var_criterion}. However, it may not be the ideal order for a mobile robot to collect the samples. For example, the user may wish to re-order the points to minimize the distance traveled by the robot. The implications of such reordering on the information gain are addressed in the subsequent section.

\subsection{Reordering Sampling Locations for Path Efficiency}
A natural question that arises with the reordering of $\mathcal{X}_T^0$ is whether the sequence in which the measurements are acquired affects the total information gain. The following lemma shows that the final information gain of a GP is invariant to the order of acquisition of measurement locations.

\begin{lemma}
\label{lem:invariance}
Let \( \mathcal{X}_{T} \subset \mathcal{X} \) be a fixed subset of \( T \) measurement locations and let \( \mathcal{X}^* \subset \mathcal{X} \) denote any arbitrary finite set of test locations. The GP posterior covariance matrix \( \Sigma_T \) computed over \( \mathcal{X}^* \) given measurements at \( \mathcal{X}_{T} \) is invariant under permutations of the acquisition order of the measurement locations.
\end{lemma}

\begin{proof}
Let $K_{**} \in \mathbb{R}^{M \times M}$ be the test-test kernel matrix with entries $[K_{**}]_{i,j} = k(x^*_i, x^*_j)$, and let $K_{*\mathrm{t}} \in \mathbb{R}^{M \times T}$ be the test-train kernel matrix with entries $[K_{*\mathrm{t}}]_{i,j} = k(x^*_i, x_j)$, where $M$ denotes the number of test locations in $\mathcal{X}^*$. Let $\mathcal{O}_A$ and $\mathcal{O}_B$ represent two distinct acquisition orderings of the set $\mathcal{X}_T$, and let $\mathbf{P} \in \mathbb{R}^{T \times T}$ be the permutation matrix that transforms ordering $\mathcal{O}_A$ to ordering $\mathcal{O}_B$. Let the kernel matrices constructed under ordering $\mathcal{O}_A$ be denoted by $K_{\mathrm{tt}}^A$ and $K_{*\mathrm{t}}^A$. Then, the corresponding matrices for ordering $\mathcal{O}_B$ satisfy
\begin{align}
K_{\mathrm{tt}}^B &= \mathbf{P} K_{\mathrm{tt}}^A \mathbf{P}^\top, \label{eq:permute_Ktt} \\
K_{\ast\mathrm{t}}^B &= K_{\ast\mathrm{t}}^A \mathbf{P}^\top, \label{eq:permute_Kat}
\end{align}
For the permuted ordering $\mathcal{O}_B$, the covariance matrix $\Sigma_T^B$ can be expressed as
\begin{align*}
\Sigma_T^B &= K_{\ast\ast} - K_{\ast\mathrm{t}}^B (K_{\mathrm{tt}}^B + \sigma_n^2 \mathbf{I}_T)^{-1} (K_{\ast\mathrm{t}}^B)^\top \\
&= K_{\ast\ast} - (K_{\ast\mathrm{t}}^A \mathbf{P}^\top) \left( \mathbf{P} K_{\mathrm{tt}}^A \mathbf{P}^\top + \sigma_n^2 \mathbf{I}_T \right)^{-1} (K_{\ast\mathrm{t}}^A \mathbf{P}^\top)^\top \\
&= K_{\ast\ast} - K_{\ast\mathrm{t}}^A \mathbf{P}^\top \left( \mathbf{P} (K_{\mathrm{tt}}^A + \sigma_n^2 \mathbf{I}_T) \mathbf{P}^\top \right)^{-1} \mathbf{P} (K_{\ast\mathrm{t}}^A)^\top.
\end{align*}
Using the identity \( (\mathbf{P} Q \mathbf{P}^\top)^{-1} = \mathbf{P} Q^{-1} \mathbf{P}^\top \) and the orthogonality of permutation matrices ($\mathbf{P}^\top \mathbf{P} = \mathbf{I}_T$), the expression simplifies to
\begin{align*}
\Sigma_T^B &= K_{\ast\ast} - K_{\ast\mathrm{t}}^A \mathbf{P}^\top \left[ \mathbf{P} (K_{\mathrm{tt}}^A + \sigma_n^2 \mathbf{I}_T)^{-1} \mathbf{P}^\top \right] \mathbf{P} (K_{\ast\mathrm{t}}^A)^\top \\
&= K_{\ast\ast} - K_{\ast\mathrm{t}}^A (\mathbf{P}^\top \mathbf{P}) (K_{\mathrm{tt}}^A + \sigma_n^2 \mathbf{I}_T)^{-1} (\mathbf{P}^\top \mathbf{P}) (K_{\ast\mathrm{t}}^A)^\top \\
&= K_{\ast\ast} - K_{\ast\mathrm{t}}^A (K_{\mathrm{tt}}^A + \sigma_n^2 \mathbf{I}_T)^{-1} (K_{\ast\mathrm{t}}^A)^\top = \Sigma_T^A.
\end{align*}
Thus, the posterior covariance $\Sigma_T$ is invariant to the acquisition order of $\mathcal{X}^0_T$. Moreover, since permutations are orthogonal, $\det(\mathbf{I}_t + \frac{1}{\sigma_n^2}{K_{\mathrm{t}\mathrm{t}}})$ is order-invariant; hence the mutual information associated with $\mathcal{X}_T$ is also invariant to acquisition order.

\end{proof}
Determining a trajectory that visits all planned measurement locations with the minimum total travel distance is the traveling salesperson problem (TSP)\cite{muzaffar.Applegate.Bixby.ea2006}, which is known to be NP-hard. For practical applications involving a moderate number of measurement locations, approximate algorithms are often sufficient. For instance, the nearest neighbor (NN) heuristic provides a computationally efficient method to construct a reasonable motion plan \cite{muzaffar.Cover.Hart1967}. In the NN algorithm, the robot proceeds from its current position to the nearest unvisited location, repeating the process until all locations in \( \mathcal{X}_T^0 \) are visited. While the NN algorithm does not guarantee a shortest path, it yields a tractable approximation suitable for real-time robotic path planning.

In this paper, the NN heuristic is employed to construct the ordered sequence of measurement locations denoted as
\begin{equation}\label{eq:xdSequence}
\mathbf{X}_T^0 = \left(x_1, x_2, \dots, x_{T}\right),
\end{equation}
where the locations are ordered such that each $x_{i+1}$ is the nearest unvisited neighbor of $x_{i}$, starting from a specified initial location $x_1$. Under Lemma~\ref{lem:invariance}, the initial sequence \( \mathbf{X}_T^0 \) in \eqref{eq:xdSequence} yields the same GP posterior covariance as the planned measurement set \( \mathcal{X}_T^0 \), while providing a feasible trajectory for the robot for data collection. 

Although $\mathbf{X}_T^0$ provides a sequence of MVS locations, the presence of unsafe regions $\mathcal{X}^u$ in the domain necessitates a safety evaluation of these measurement locations. To facilitate the safety evaluation from the evolving GP posterior, the next section focuses on detecting high-intensity regions of the field, which are then used to modify the measurement plan.

\section{Estimation of High Intensity Regions using Hough Transform}\label{sec:hough_transform_estimation}
To ensure the safety of the robot during the measurement process, we must evaluate the uncertainty associated with the GP posterior. Since the true function $f$ is unknown, we rely on a probabilistic confidence bound to distinguish safe and unsafe regions.

Evaluating the GP posterior error over the infinite set $\mathcal{X}$ is analytically intractable. A commonly used approach is to evaluate the GP posterior on a finite number of test locations (cf. \cite{SCC.Wendland2004, SCC.Rasmussen.Williams2006}). To facilitate the analysis, let the subset $\mathcal{X}^* \subset \mathcal{X}$ be a finite set of uniformly distributed testing locations in $\mathcal{X}$, where $|\mathcal{X^*}|=M$. The finite $\mathcal{X}^*$ is quantified by two geometric measures, which are the fill distance \( h_{\mathcal{X}^*, \mathcal{X}} \) and the separation radius \( q_{\mathcal{X}^*} \), defined as

\begin{align}
h_{\mathcal{X}^*, \mathcal{X}} &\coloneqq \sup_{x \in \mathcal{X}} \min_{x^* \in \mathcal{X}^*} \|x - x^*\|_2, \label{eq:fill_distance} \\
q_{\mathcal{X}^*} &\coloneqq \min_{x_i^* \neq x_j^*} \|x_i^* - x_j^*\|_2. \label{eq:separation_radius}
\end{align}
To determine whether a location lies within the safe region \( \mathcal{X}^s \), the associated uncertainty is evaluated using the GP posterior at each update step \( t \). The following lemma from \cite{muzaffar.Srinivas.Krause.ea2012} establishes uniform confidence bounds which will be utilized for both high-intensity region estimation and evaluation of planned measurement locations for safety.

\begin{lemma}\cite[Lemma 5.1]{muzaffar.Srinivas.Krause.ea2012}\label{lem_2}
Given a domain $\mathcal{X}^* \subset \mathbb{R}^n$, a confidence level \( \delta \in (0,1) \), and a non-decreasing sequence \( \{ \pi_t \}_{t \in \mathbb{N}} \) such that \( \sum_{t \geq 1} \pi_t^{-1} = 1 \), if 
\begin{equation}\label{eq:beta_function_MCB}
\beta_t \coloneqq 2 \ln\left( \frac{M \, \pi_t}{\delta} \right),
\end{equation}
then, with probability at least \( 1 - \delta \),
\begin{equation}\label{eq:gp_confidence_bound_lemma1}
 |f(x) - \mu_{t-1}(x)| \le \sqrt{\beta_t} \, \sigma_{t-1}(x), \, \forall t \leq T, \, \forall x \in \mathcal{X}^*,
\end{equation}
where \( \mu_{t-1} \) and \( \sigma_{t-1}\) are the GP posterior mean and covariance functions computed using the observed data \( \mathcal{D}_{t-1} \).
\end{lemma}
\subsection{Binary Safety Map and Hough Transform}\label{sec:binary_hough_transform_estimation}

To detect and identify the high-intensity regions using the GP posterior mean \( \mu_{t-1} \) and covariance \( \sigma_{t-1} \) at any measurement step \(t-1 \), a binary map \( \mathcal{G}_t : \mathcal{X}^* \to \{0,1\} \) is generated using
\begingroup\medmuskip=0mu\begin{equation}
\mathcal{G}_t(x^*) =
\scalebox{0.9}{$\begin{cases}
0, & \text{if } \mu_{t-1}(x^*) \ + \sqrt{\beta_{t}} \sigma_{t-1}(x^*) + L_f h_{\mathcal{X}^*,\mathcal{X}} \le \bar{f}, \\
1, & \text{otherwise},
\end{cases}$}
\label{eq:binary_map}
\end{equation}\endgroup
where $\bar{f}$ is the user-defined safety threshold, \( L_f \) is the Lipschitz constant defined in Assumption \ref{ass:Lip_bound}, and $\beta_{t}$ is the confidence parameter defined in Lemma \ref{lem_2}. 

The HT is applied to the data generated by the binary map \( \mathcal{G}_t \) at each measurement step \( t \). The HT is a technique used in computer vision to isolate geometric primitives by transforming spatial points into curves within a parameter space; shapes are identified by parameter combinations accumulating the maximum number of consistent votes \cite{muzaffar.Duda.Hart1972, muzaffar.Yuen.Princen.ea1990}.

In this paper, the circular or spherical Hough transform is utilized (cf.~\cite{muzaffar.Qureshi.Ogri.ea2024}) to detect objects in the binary image in \eqref{eq:binary_map} yielding
\begin{equation}
\mathcal{C}_t \coloneqq \left\{ (c_i(t), r_i(t)) \right\}_{i=1}^{P_t},
\label{eq:hough_output_2D}
\end{equation}
where \( c_i(t) \) and \( r_i(t)\) denote the center and radius of the \( i \)-th detected circle or sphere, and \( P_t \) denotes the number of circles or spheres identified at iteration \( t \). The union of these detected geometries yields an approximation of the unsafe region, defined as
\begin{equation}
\bar{\mathcal{X}}^u_t \coloneqq \bigcup_{i=1}^{P_t} \bar{\mathcal{B}}(c_i(t), r_{i}(t)),
\label{eq:overapprox_unsafe}
\end{equation}
where \( \bar{\mathcal{B}}(c, r) \) denotes the closed ball of radius $r$ centered at $c$. The estimate of the safe region is given by
\begin{equation}
\bar{\mathcal{X}}^s_t \coloneqq \mathcal{X} \setminus \bar{\mathcal{X}}^u_t.
\label{eq:overapprox_safe}
\end{equation}
The estimates \( \bar{\mathcal{X}}^s_t \) and \( \bar{\mathcal{X}}^u_t \) guide the safety-constrained relocation of planned measurement locations and the synthesis of collision-free trajectories between any pair of measurement locations.
 \begin{remark}
Derivation of bounds on the error between the true safe set \( \mathcal{X}^s \) and the estimated safe set $\bar{\mathcal{X}}^s_t$ is precluded by the fact that the level sets of the underlying scalar field $f$ are not assumed to follow any predefined parametric forms or fixed geometric primitives, and is out of the scope of the paper.
 \end{remark}

\subsection{Safety Evaluation of Measurement Locations}
To avoid measurements in the unsafe set, at each step $t$, the GP posterior generated using the dataset $\mathcal{D}_t$ is used to evaluate the remaining $T-t$ planned locations for safety, and replanned if necessary. Let $\mathbf{X}_T^t$ denote the updated measurement location set at time step $t$ consisting of visited measurement locations up to time step $t$ and planned measurement locations from $t+1$ to $T$.

To determine whether any planned measurement location in $\mathbf{X}_T^t$ lies within the safe region, their associated uncertainty is evaluated using the confidence bounds derived in Lemma \ref{lem_2}. Given the user-defined safety threshold \( \bar{f} \), the time-varying safe subset $\mathbf{X}^{s,t}$ is defined using the probabilistic safety criterion~\eqref{eq:gp_confidence_bound_lemma1} as
\begingroup\medmuskip=0mu\thinmuskip=2mu\thickmuskip=2mu\begin{equation}\label{eq:estimated_safe_region}
\scalebox{0.97}{$\mathbf{X}^{s,t} = \left\{ x \in \mathbf{X}_{T}^{t} \;\middle|\; \mu_{t}(x) + \sqrt{\beta_t} \, \sigma_{t}(x) + L_f h_{\mathcal{X}^*,\mathcal{X}} \le \bar{f}\right\}$},
\end{equation}\endgroup
for $t \geq 1$. At any measurement step $t$, \eqref{eq:estimated_safe_region} provides the safe subset with probabilistic safety guarantees. To utilize the measurement budget $T$, locations deemed unsafe, i.e., the ones in $\mathbf{X}^t_T \setminus \mathbf{X}^{s,t}$ are replanned to render them safe. The following subsection outlines the relocation process.

\subsection{Relocation of Measurement Locations}\label{sec:safe_relocation_replenishment}
Any future location found to lie within the estimated unsafe region \( \bar{\mathcal{X}}_t^u \), as identified by the HT, is replaced with a new location in $\bar{\mathcal{X}}^s_t$. The replacement location is selected from \( \bar{\mathcal{X}}_t^s \) to minimize its distance from the initially planned unsafe location. Relocation of all locations in $\mathbf{X}^t_T \setminus \mathbf{X}^{s,t}$ results in a new planned set $\mathbf{X}_{T}^{t+1}$. Note that \( \mathbf{X}_T^{T} \) denotes the true set of executed measurement locations at the conclusion of the mapping experiment.

The classification of safe and unsafe regions depends on the GP posterior, which evolves as measurement data accumulate \cite{SCC.Rasmussen.Williams2006}. Hough transform-based unsafe region estimates derived early in the process may result in false positives, where safe regions are incorrectly classified as unsafe or vice versa. These misclassifications can lead to the formation of unvisited pockets within the safe domain, hindering the objective of maximizing exploration in safe regions. To correct these misclassifications, the algorithm recursively evaluates the set of previously rejected planned locations. Any location from the initial set $\mathbf{X}_T^0$ that was previously deemed unsafe but is subsequently found to lie within the updated safe region $\bar{\mathcal{X}}^s_t$ is reinstated into the set of planned measurement locations. This strategy ensures that the robot corrects for initial conservatism as the model uncertainty diminishes.

\begin{remark}This paper focuses on the finite measurement budget setting, where budget constraints necessitate a minimum-distance based relocation strategy. This relocation strategy aims to minimize deviations from the planned robot trajectory while effectively utilizing boundary information between safe and unsafe regions from HT. In contrast, under an infinite measurement budget setting, one could employ a maximum-variance sampling approach by iteratively solving Eq. \eqref{eq:var_criterion} over the estimated safe set $\bar{\mathcal{X}}^s_t$. As the number of iterations $T \to \infty$, this latter strategy would yield a dense sampling of the entire safe domain $\mathcal{X}^s$.
\end{remark}

Algorithm \ref{alg:gp_ht} details the complete GP-HT framework, including the relocation strategy and retroactive GP updates. The subsequent section introduces a method for designing feasible trajectories that ensure the robot avoids unsafe regions during the mapping process, leveraging the HT-based estimates of $\bar{\mathcal{X}}^u$.

\subsection{Safe Trajectory Generation using RRT*}
\label{sec:rrt_star_paths}
In this section, the path planning problem is addressed to determine an optimal trajectory that enables the robot to visit the planned measurement locations while avoiding high-intensity regions. Trajectory planning in environments with obstacles is a classical problem in robotics (e.g. \cite{muzaffar.Katona.Neamah.ea2024,muzaffar.Kunchev.Jain.ea2006,muzaffar.Minguez.Lamiraux.ea2008}). Approaches to this problem are generally categorized into optimization-based methods, such as CHOMP~\cite{muzaffar.Ratliff.Zucker.ea2009}, STOMP~\cite{muzaffar.Kalakrishnan.Chitta.ea2011}, and TrajOpt~\cite{muzaffar.Schulman.Ho.ea2013}, which iteratively refine trajectories using gradient information, and sampling-based algorithms, such as the rapidly-exploring random tree (RRT) and its asymptotically optimal variant, RRT*~\cite{muzaffar.Karaman.Frazzoli2011}.

Given a finite sequence of measurement locations \( \mathbf{X}_T^T = (x_1, x_2, \ldots, x_T) \subset \mathcal{X} \), the objective is to construct a safe path for a sensor-equipped robot to traverse between two successive locations and navigate the entire planned trajectory. In this work, at any measurement step \( t \), the estimated unsafe region \( \bar{\mathcal{X}}^u_t \) is treated as the obstacle map, and the RRT* algorithm  is utilized to construct a collision-free path \( \phi_t \) from \( x_t \) to \( x_{t+1} \). This process is repeated until all measurement locations have been visited. This trajectory planning approach guarantees safety subject to the accuracy of the Hough transform-based high-intensity region approximation.

The subsequent section presents a quantitative analysis of the information loss to compare the performance of the proposed GP-HT sampling strategy against the baseline MVS approach, specifically comparing the execution of measurement sequence $\mathbf{X}_T^T$ versus $\mathbf{X}_T^0$.

\section{Information Loss due to GP-HT Algorithm}\label{sec:information_loss_Gp_HT}
To quantify the effect of relocating planned measurement locations on the information gain, the information gain obtained from the final executed set \( \mathbf{X}^T_T \), is compared to that of the initially planned greedy set \( \mathbf{X}^0_T \). Given any sequence of $T$ measurement locations, the information gain can be expressed in terms of the eigenvalues of the kernel matrix derived from \eqref{eq:cov_matrix_eqn} as
\begin{multline}\label{eq:sequence_points}
I(\mathcal{Y}; F_{\mathbf{X}_T}) = \tfrac{1}{2} \ln\left(\det( \mathbf{I}_T + \sigma_n^{-2} K_{\mathrm{t}\mathrm{t}}) \right)
\\= \tfrac{1}{2} \sum_{i=1}^{T} \ln\left(1 + \tfrac{\lambda_i}{\sigma_n^2}\right),
\end{multline}
where $\lambda_i$ denotes the $i$-th eigenvalue of the kernel matrix $K_{\mathrm{t}\mathrm{t}}$ \cite{muzaffar.Srinivas.Krause.ea2012} . Let $K^g \in \mathbb{R}^{T\times T}$ denote the kernel matrix computed using \eqref{eq:Ktt}, using the greedy sampling locations in $\mathbf{X}_T^0$. Let $K^s \in \mathbb{R}^{T\times T}$ denote the kernel matrix computed using the executed measurement locations in $\mathbf{X}_T^T$. Let \( \{\lambda_i^g\} \) and \( \{\lambda_i^s\} \) denote the set of eigenvalues of \( K^g \) and \( K^s \), respectively. The difference in the information gain due to safety-aware relocation can then be computed as \cite{muzaffar.Nemhauser.Wolsey.ea1978}
\begingroup\medmuskip=0mu\thinmuskip=2mu\thickmuskip=2mu
\begin{align}\label{eq:delta_gamma}
\Delta \gamma 
\coloneqq \gamma_T^g - \gamma_T^s 
= \frac{1}{2} \sum_{i=1}^{T} 
\left[ \ln\!\left(1 + \frac{\lambda_i^g}{\sigma_n^2} \right) 
- \ln\!\left(1 + \frac{\lambda_i^s}{\sigma_n^2} \right) \right].
\end{align}
\endgroup
Using the definition of $\gamma_T^g$ from \eqref{eq:gamma_t_greedy} and its correlation from optimal information value $\gamma_T^o$ from \eqref{eq:submodular_bound}, we get 
\begin{equation}
\gamma_T^g \geq \left( 1 - \frac{1}{e} \right) \gamma_T^o,
\end{equation}
The $\Delta \gamma$ from \eqref{eq:delta_gamma} can be utilized to bound the safe information gain
\begin{equation}\label{eq:gamma_s_lower_via_delta}
\gamma_T^s = \gamma_T^g - \Delta\gamma
\ge \Big(1-\frac{1}{e}\Big)\gamma_T^o - \Delta\gamma,
\end{equation}
where $\gamma_T^s$ can be computed using \eqref{eq:cov_matrix_eqn} and final set of executed measurement locations $\mathbf{X}_T^T$ using GP-HT sampling strategy. The loss of information gain due to $\mathbf{X}^T_T$ can be compared using the inequality (see \cite{muzaffar.Nemhauser.Wolsey.ea1978}), 
\begin{equation}
\label{eq:gamma_safe_lower_bound}
\gamma_T^s \geq \frac{\gamma_T^s}{\gamma_T^g} \left( 1 - \frac{1}{e} \right) \gamma_T^o.
\end{equation}

Using the replenishment scheme of adding safe measurement locations in the planned set as discussed in Section \ref{sec:safe_relocation_replenishment}, one can conclude that the loss of information $\Delta \gamma$ is only attributed to unsafe planned measurement locations, while preserving the information utility of safe locations. The bound in \eqref{eq:gamma_safe_lower_bound} provide a quantitative analysis of the information loss while maintaining safety during the measurement process.

The next section provides convergence guarantees for the evolving GP posterior under the GP-HT sampling strategy. The analysis is required since the domain for selecting measurement locations is restricted to safe regions.

\section{Convergence Analysis of the GP-HT Algorithm}
\label{sec:convergence_analysis}
This section establishes convergence properties of the developed GP-HT algorithm under safety constraints. The analysis examines a GP-based safety condition defined on a finite test grid and its extension to the continuous domain via Lipschitz regularity. The objective is to show that unsafe locations are not certified as safe with high probability, and that interior safe locations are eventually certified as the number of measurements increases and the GP posterior uncertainty contracts.

While the practical implementation uses minimum-distance relocation to exploit safe-unsafe boundary information for robotic path planning, the analysis considers the idealized setting where measurement budget $T \to \infty$ and locations are selected using MVS on the certified safe set. In this setting, the MVS rule serves as an exploration heuristic that samples in regions of high covariance, thereby accelerating variance contraction and enabling asymptotic identification of the safe region.

All stochastic statements below are understood with respect to an underlying probability space $(\Omega,\mathcal{F},\mathbb{P})$, where $\Omega$ denotes the set of all realizations of the measurement-noise sequence in the GP observation model, $\mathcal{F}$ is the associated $\sigma$-algebra, and $\mathbb{P}$ is the probability measure induced by the assumed noise distribution. Accordingly, the event $\mathcal{E}\in\mathcal{F}$ represents the subset of noise realizations for which the GP confidence bound holds uniformly over the test grid and over all time steps. Using Lemma \ref{lem_2}, the subsequent analysis relies on the high-probability event defined as
\begin{equation}\label{eq:confidence_event}
\scalebox{0.88}{$
\mathcal{E}
\coloneqq
\left\{
\omega\in\Omega
\;\middle|\;
\forall x\in\mathcal{X}^*,\ \forall t\ge 1:
\bigl| f(x) - \mu_{t-1}(x) \bigr|
\le
\sqrt{\beta_t}\,\sigma_{t-1}(x)
\right\}.
$}
\end{equation}
Provided \(\mathcal{E}\) occurs, the GP confidence bounds hold uniformly for every grid point and for all time steps, allowing the subsequent convergence analysis to be carried out deterministically. It is then shown that the safety condition never admits a grid point whose true function value exceeds the safety threshold.

To formalize this, the true safe and unsafe subsets of the test grid $\mathcal{X}^*$ are defined as
\begin{equation}
\mathcal{X}^{s,*}
\coloneqq
\{x\in\mathcal{X}^* \mid f(x)\le \bar f\},
\end{equation}
\begin{equation}
\mathcal{X}^{u,*}
\coloneqq
\{x\in\mathcal{X}^* \mid f(x)>\bar f\}.
\end{equation}
Using the GP confidence bound specified in \eqref{eq:confidence_event}, define the GP-based safe set as
\begin{equation}
\widehat{\mathcal{X}}^{s,*}_t
\coloneqq
\left\{
x\in\mathcal{X}^*
\ \middle|\
\mu_{t-1}(x)+\sqrt{\beta_t}\,\sigma_{t-1}(x)
\le \bar f
\right\}.
\label{eq:disc_safe_hat_conv}
\end{equation}
By construction, a grid point is certified as safe only if its GP upper confidence bound lies below the prescribed safety threshold. The following lemma establishes the correctness of the safety condition in \eqref{eq:disc_safe_hat_conv}, showing that no true unsafe location is assumed safe.
\begin{lemma}\label{lem:disc_soundness}
Provided \(\mathcal{E}\) occur, the GP-based estimate of the safe set satisfies
\begin{equation}
\widehat{\mathcal{X}}^{s,*}_t \subseteq \mathcal{X}^{s,*},
\qquad \forall t \ge 1.
\end{equation}
\end{lemma}
\begin{proof}
For an arbitrary time index \( t \ge 1 \), let \( x \in \widehat{\mathcal{X}}^{s,*}_t \) be given. By the definition of the discrete GP-based safety condition,
\begin{equation}
\mu_{t-1}(x) + \sqrt{\beta_t}\,\sigma_{t-1}(x) \le \bar f.
\end{equation}
On $\mathcal{E}$, the GP confidence property gives
\begin{equation}
f(x) \le \mu_{t-1}(x) + \sqrt{\beta_t}\,\sigma_{t-1}(x).
\end{equation}
Combining the two inequalities yields $f(x) \le \bar f$, and therefore $x \in \mathcal{X}^{s,*}$.
\end{proof}
Lemma~\ref{lem:disc_soundness} establishes that the discrete GP-based safety condition never admits any grid point whose true function value 
violates the safety constraint. This guarantee, however, applies only to the finite set of test locations. To extend the safety guarantee to the continuous domain, the result is lifted using the assumed Lipschitz regularity of the unknown function $f$. To this end, the nearest-neighbor projection onto the test grid is defined as
\begin{equation}
\Pi_{\mathcal{X}^*}(x)\in
\arg\min_{x^*\in\mathcal{X}^*}\|x-x^*\|_2,
\end{equation}
which satisfies
\begin{equation}
\|x-\Pi_{\mathcal{X}^*}(x)\|_2
\le h_{\mathcal{X}^*,\mathcal{X}},
\qquad \forall x\in\mathcal{X}.
\end{equation}
This projection enables the construction of a continuous-domain safety condition by augmenting a Lipschitz margin that accounts for interpolation error between grid points.

Under Assumption~\ref{ass:Lip_bound} on the Lipschitz continuity of $f$, we define the continuous GP-based safe set as
\begin{equation}
\scalebox{0.95}{$\widehat{\mathcal{X}}^{s}_t
\coloneqq
\left\{
x\in\mathcal{X}
\ \middle|\
\mu_{t-1}(x^*)
+\sqrt{\beta_t}\,\sigma_{t-1}(x^*)
+L_f h_{\mathcal{X}^*,\mathcal{X}}
\le \bar f
\right\},$}
\label{eq:expanded_safe_via_proj_conv}
\end{equation}
where $x^* = \Pi_{\mathcal{X}^*}(x)$ denotes the nearest test-grid point to \(x\). The following result formalizes the safety condition, showing that the Lipschitz-based lifting preserves the safety guarantees established on the discrete test grid.
\begin{theorem}\label{thm:continuous_safety}
Let the Assumption~\ref{ass:Lip_bound} hold and that the GP confidence event \(\mathcal{E}\) defined in \eqref{eq:confidence_event} occur.
Then, for all \( t \ge 1 \), the GP-based safety condition defined in \eqref{eq:expanded_safe_via_proj_conv} satisfies
\begin{equation}
\widehat{\mathcal{X}}^{s}_t \subseteq \mathcal{X}^{s}.
\end{equation}
\end{theorem}
\begin{proof}
Let $x\in\widehat{\mathcal{X}}^{s}_t$ and $x^*=\Pi_{\mathcal{X}^*}(x)$, the safety condition can be written as 
\begin{equation}
\mu_{t-1}(x^*)+\sqrt{\beta_t}\,\sigma_{t-1}(x^*)
+L_f h_{\mathcal{X}^*,\mathcal{X}}
\le \bar f.
\end{equation}
On the event $\mathcal{E}$, we have $f(x^*)\le \mu_{t-1}(x^*)+\sqrt{\beta_t}\,\sigma_{t-1}(x^*)$. Therefore, we have
\begin{equation}
f(x^*)+L_f h_{\mathcal{X}^*,\mathcal{X}}\le \bar f.
\end{equation}
By Lipschitz continuity of $f$,
\begin{equation}
f(x)\le f(x^*)+L_f\|x-x^*\|_2
\le \bar f,
\end{equation}
which implies $x\in\mathcal{X}^{s}$.
\end{proof}
Assuming the high-probability confidence event \(\mathcal{E}\) occur, every location admitted into the condition is truly safe, so no unsafe point is ever certified as safe. However, a location may be safe while still failing to satisfy \eqref{eq:disc_safe_hat_conv} if the GP upper confidence bound remains above the safety threshold. Safe points are likely to fail \eqref{eq:disc_safe_hat_conv} near the safety boundary, where $f(x) \approx \bar{f}$, and thus small posterior uncertainty can prevent safety certification.
To obtain a meaningful asymptotic certification result, the analysis therefore focuses on grid points that lie strictly inside the safe region with a nonzero margin separating them from the threshold. Accordingly, for any $\eta > 0$, the interior-safe grid is defined as
\begin{equation}
\mathcal{X}^{s,*,\eta}
\coloneqq
\left\{
x\in\mathcal{X}^*
\ \middle|\
f(x)\le \bar f - L_f h_{\mathcal{X}^*,\mathcal{X}} - \eta
\right\}.
\end{equation}The $\eta$ margin excludes points whose safety status may be affected by discretization or boundary effects. We now show that interior-safe grid points are eventually included in the discrete safety condition, provided that posterior uncertainty at those points contracts sufficiently.
\begin{theorem}\label{thm:interior_cert}
Let $\mathcal{X}^*\subset\mathcal{X}$, let $k$ be positive definite, and assume Gaussian noise with variance $\sigma_n^2>0$.
Let $\{\sigma_t^2(\cdot)\}_{t\ge 0}$ denote the GP posterior variance sequence induced by any sampling sequence
$\{x_t\}_{t\ge 1}\subseteq \mathcal{X}^*$. Then the following holds.
\begin{enumerate}[label=(\alph*)]
\item For every $z\in \mathcal{X}^*$, the sequence $\{\sigma_t^2(z)\}_{t\ge 0}$ is non-increasing and convergent.
\item If a point $x\in \mathcal{X}^*$ is sampled infinitely often, then
\begin{equation}
\sigma_t^2(x)\to 0
\qquad \text{as } t\to\infty.
\end{equation}
\item On the event $\mathcal{E}$, fix any margin $\eta>0$ and any
$x\in \mathcal{X}^{s,*,\eta}$.
If
\begin{equation}
\lim_{t\to\infty}\sqrt{\beta_t}\,\sigma_{t-1}(x) \rightarrow 0,
\end{equation}
then there exists a finite time $T_x<\infty$ such that
\begin{equation}\label{eq:safety_condition}
x\in \widehat{\mathcal{X}}^{s,*}_t,
\qquad \forall t\ge T_x.
\end{equation}
\end{enumerate}
\end{theorem}
\begin{proof}
\emph{(a).} The GP posterior variance update under noisy observations is given by
\begin{equation}
\sigma_t^2(z)
=
\sigma_{t-1}^2(z)
-
\frac{\Sigma_{t-1}(z,x_t)^2}{\sigma_n^2+\sigma_{t-1}^2(x_t)},
\qquad \forall z\in\mathcal{X}^*.
\end{equation} The subtracted term is nonnegative for all $z$, so $\sigma_t^2(z)$ is nonincreasing.
Since $\sigma_t^2(z)\ge 0$, the sequence converges.
\newline
\emph{(b).}
Let $\{t_j\}$ denote the subsequence of times at which $x_{t_j}=x$.
Suppose $\lim_{t\to\infty}\sigma_t^2(x)=v>0$.
Then for sufficiently large $j$,
\begin{equation}
\frac{\Sigma_{t_j-1}(x,x)^2}{\sigma_n^2+\sigma_{t_j-1}^2(x)}
=
\frac{\sigma_{t_j-1}^4(x)}{\sigma_n^2+\sigma_{t_j-1}^2(x)}
\ge
\frac{(v/2)^2}{\sigma_n^2+v},
\end{equation}
which is a strictly positive constant.
This contradicts the convergence of $\sigma_t^2(x)$.
Hence $\sigma_t^2(x)\to 0$.
\newline
\emph{(c).}
For $\eta>0$ and $x\in\mathcal{X}^{s,*,\eta}$, we have
\begin{equation}
f(x)\le \bar f - L_f h_{\mathcal{X}^*,\mathcal{X}} - \eta.
\end{equation}
On the event $\mathcal{E}$, the GP confidence bound yields
\begin{equation}
    f(x)
    \;\ge\;
    \mu_{t-1}(x) - \sqrt{\beta_t}\,\sigma_{t-1}(x),
\end{equation}
Adding $\sqrt{\beta_t}\,\sigma_{t-1}(x)$ to both sides gives
\begin{equation}
    \mu_{t-1}(x) + \sqrt{\beta_t}\,\sigma_{t-1}(x)
    \;\le\;
    f(x) + 2\sqrt{\beta_t}\,\sigma_{t-1}(x).
\end{equation}
Substituting the interior-safe inequality for $f(x)$ yields
\begin{equation}
    \mu_{t-1}(x) + \sqrt{\beta_t}\,\sigma_{t-1}(x)
    \;\le\;
    \bar f - L_f\, h_{\mathcal{X}^\ast,\mathcal{X}} - \eta
    + 2\sqrt{\beta_t}\,\sigma_{t-1}(x).\label{eq:neta_safety_condition1}
\end{equation}
Assuming $\lim_{t\to\infty} \sqrt{\beta_t}\,\sigma_{t-1}(x) = 0$ and by the definition of the limit, there exists $T_x < \infty$ such that for all $t \ge T_x$,
\begin{equation}
    2\sqrt{\beta_t}\,\sigma_{t-1}(x) \;\le\; \eta.\label{eq:neta_safety_condition}
\end{equation}
This condition requires that the posterior uncertainty at $x$ contracts sufficiently fast relative to the growth of $\beta_t$, which is determined by the chosen GP confidence parameter. Substituting \eqref{eq:neta_safety_condition} into \eqref{eq:neta_safety_condition1} implies that for all $t \ge T_x$,
\begin{equation}
    \mu_{t-1}(x) + \sqrt{\beta_t}\,\sigma_{t-1}(x)
    \;\le\;
    \bar f - L_f\, h_{\mathcal{X}^\ast,\mathcal{X}}.
\end{equation}
In particular, we have
\begin{equation}
    \mu_{t-1}(x) + \sqrt{\beta_t}\,\sigma_{t-1}(x)
    \;\le\;
    \bar f,
\end{equation}
and therefore, \eqref{eq:disc_safe_hat_conv} implies \eqref{eq:safety_condition}.
\end{proof}

\begin{corollary}\label{cor:interior_complete}
On the event \(\mathcal{E}\), fix any margin \(\eta>0\). If
\begin{equation}
\sqrt{\beta_t}\,\sigma_{t-1}(x)\to 0,
\qquad \forall x\in\mathcal{X}^{s,*,\eta},
\end{equation}
then there exists a finite time \(T^\ast<\infty\) such that
\begin{equation}\label{eq:subset}
\mathcal{X}^{s,*,\eta}\subseteq \widehat{\mathcal{X}}^{s,*}_t,
\qquad \forall t\ge T^\ast.
\end{equation}
\end{corollary}
\begin{proof}
By Theorem~\ref{thm:interior_cert}(c), for each \(x\in\mathcal{X}^{s,*,\eta}\), there exists a finite time \(T_x<\infty\) such that $x\in \widehat{\mathcal{X}}^{s,*}_t$ for all $t\ge T_x$.

Since \(\mathcal{X}^{s,*,\eta}\subseteq\mathcal{X}^*\) and \(\mathcal{X}^*\) is finite, the set \(\mathcal{X}^{s,*,\eta}\) is finite. Hence, $ T^\ast$ is defined as $T^\ast \coloneqq \max_{x\in\mathcal{X}^{s,*,\eta}} T_x < \infty$
Thus, \eqref{eq:subset} holds.
\end{proof}

In GP-HT, sampling is restricted to the currently certified-safe region and is guided by the MVS rule, which repeatedly queries locations of highest posterior variance within the admissible set. This safe-restricted variance-seeking rule is reminiscent of space-filling and greedy variance-reduction strategies studied in kernel interpolation, such as \(P\)-greedy sampling \cite[Thm.~4.1]{AlbrechtIske2025}. Theorem~\ref{thm:interior_cert} shows that any interior-safe grid point is eventually certified once its confidence width becomes sufficiently small, and Corollary~\ref{cor:interior_complete} then shows that there exists a finite time after which all points in \(\mathcal{X}^{s,*,\eta}\) belong to the certified-safe set \(\widehat{\mathcal{X}}^{s,*}_t\). In addition, Theorem~\ref{thm:continuous_safety} guarantees that the Lipschitz-lifted set \(\widehat{\mathcal{X}}_t^s\) remains a sound inner approximation of the true continuous safe set \(\mathcal{X}^s\) for all \(t\). Consequently, the proposed method asymptotically certifies a corresponding interior subset of the continuous safe region through the Lipschitz lifting, up to the discretization margin \(L_f h_{\mathcal{X}^*,\mathcal{X}}\).

\section{Simulations and Results}\label{sec:simulations}
This section presents two simulation studies and a hardware experiment to evaluate the performance of the developed safety-aware scalar field mapping framework using the GP-HT algorithm (Algorithm~\ref {alg:gp_ht}). Simulation 1 (section~\ref{subsec:simulation1}) considers a 2D scalar field with multiple high-intensity regions distributed across the domain. Simulation 2 (section~\ref{subsec:simulation2}) extends the framework to a 3D setting, where high-intensity regions are modeled as spherical regions. Finally, a hardware experiment demonstrates the effectiveness of the GP-HT algorithm using a wheeled robot equipped with an onboard scalar field intensity sensor (section~\ref{subsec:Experimental}).

\subsection{Two-Dimensional Simulation}\label{subsec:simulation1}
In this simulation, the scalar field used is generated using 
\begin{align}
f(x, y) = &\; \exp\left( -\frac{(x - x_1)^2 + (y - y_1)^2}{0.08} \right) \notag \\
         & + \exp\left( -\frac{(x - x_2)^2 + (y - y_2)^2}{0.08} \right),
\label{eq:true_field}
\end{align}
where the coordinates of the two scalar field sources are selected as $ (x_1, y_1) = (0.25,0.75)$ and $ (x_2, y_2) = (0.75,0.25) $. The domain of the function is restricted to $ x, y \in [0, 1]$. The function output ranges in $f \in (0,1]$, and the threshold $\bar{f}=0.7$ is selected to define high-intensity regions. The exponential term in \eqref{eq:true_field} represents a localized intensity peak centered at its respective source location, with the parameter 0.08 controlling the spatial spread, as shown in Figure \ref{fig:SurfacePlot}. Initially, 100 measurement locations are planned using the information maximization in \eqref{eq:var_criterion} as shown in Figure \ref{fig:temporal_GP-HT-2D}.

The GP model is initialized with a zero-mean prior,  \( \alpha =1  \text{ and } l = 0.15 \) are selected as hyperparameters for the kernel. The GP posterior is evaluated over a fixed grid of 10,000 uniformly distributed testing locations. High-intensity regions are estimated at each step using the \texttt{imfindcircles} HT routine of MATLAB on the binary high-intensity map generated from the GP posterior. A radius range of $[0.05, 0.15]$ is selected to capture the high-intensity regions, with a sensitivity parameter of 1 to ensure robust detection. A temporal evaluation of the detected high-intensity regions is illustrated in Figure~\ref{fig:temporal_GP-HT-2D}. Figures~\ref{fig:temporal_GP-HT-2D} also compare the initially planned and the final executed measurement locations.

\begin{algorithm}[ht]
\caption{GP-HT: Safe field mapping using Hough transform-based high-intensity detection}
\label{alg:gp_ht}
\begin{algorithmic}[1]
\Require GP hyperparameters \( (\alpha, l, \sigma_n) \), safety threshold \( \bar{f} \), initial location of the robot \( x_1 \), testing locations \( \mathcal{X}^* \)
\Ensure Final dataset $\mathcal{D}_T$, posterior GP $\mu_T$ and $\sigma_N$, executed measurement sequence $\mathbf{X}_T^T$, trajectory $\Phi$
\vspace{2mm}
\State \textbf{Offline:}
\State Given intial location $x_1$ and $\mathcal{X}^*$ compute kernel matrices using \eqref{eq:Ktt}
\State Select \( T \) locations via  MVS samping over \( \mathcal{X}^* \) using \eqref{eq:var_criterion}  and compute $\mathcal{X}^0_T$
\State Reorder \( \mathcal{X}^0_T \) via nearest-neighbor strategy to minimize travel distance and compute  $\mathbf{X}^0_T$
\State \textbf{Initialize:} Set executed sequence $\mathbf{X}_T^T \gets \emptyset$, dataset $\mathcal{D}_0 \gets \emptyset$, and global trajectory $\Phi \gets \emptyset$, prior mean $\mathbf{\mu}_0 \gets \mathbf{0}_{M \times 1}$ and covariance $\mathbf{\Sigma}_0 \gets \mathbf{K}_{\mathcal{X}^*, \mathcal{X}^*}$

\State \textbf{Online:}
\For{$t = 0, \dots, T$}
    \State Observe $y_t = f(x_t) + \epsilon$ and update history: $\mathcal{D}_t \gets \mathcal{D}_{t-1} \cup \{(x_t, y_t)\}$ and $\mathbf{X}_T^T \gets \mathbf{X}_T^T \cup \{x_t\}$
    \State Compute $\mathbf{\mu}_t \in \mathbb{R}^M, \mathbf{\Sigma}_t \in \mathbb{R}^{M \times M}$ using Eq.~\eqref{eq:mean_eq} and \eqref{eq:var_eq}
    \State Generate binary map \( \mathcal{G}_t \) using \( \mathbf{\mu}_t, \mathbf{\Sigma}_t, \beta_t, \bar{f} \) (Eq.~\ref{eq:binary_map})
    \State Apply HT on \( \mathcal{G}_t \) to extract primitives $\mathcal{C}_t$ (Eq.~\ref{eq:hough_output_2D}) using functions \texttt{imfindcircles} or \texttt{regionprops3} to 
     \State Construct unsafe region approximation $\bar{\mathcal{X}}_t^u$ (Eq.~\ref{eq:overapprox_unsafe}) and safe region $\bar{\mathcal{X}}_t^s$ (Eq.~\eqref{eq:hough_output_2D})
    \State Compute $\bar{\mathcal{X}}_t^u$ (Eq.~\eqref{eq:overapprox_safe})
    \For{$k = t+1, \dots, T$} \Comment{Check future planned points}
        \If{$x_k \in \bar{\mathcal{X}}^u_{t}$}
            \State Relocate $x_k \gets \arg \min_{x \in \bar{\mathcal{X}}^s_t} \| x - x_k \|_2$
        \EndIf
    \EndFor
    \State Check blocked candidates from $\mathbf{X}_T^0$ reinstate if now in $\bar{\mathcal{X}}^s_t$
    \If{$t < T$}
        \State Compute local path $\phi_t$ from $x_t$ to $x_{t+1}$ using RRT* with $\bar{\mathcal{X}}^u_{t}$ as obstacles
        \State Execute $\phi_t$ to move robot to $x_{t+1}$
        \State Append local path to global trajectory: $\Phi \gets \Phi \cup \phi_t$
    \EndIf
\EndFor
\State \Return $\mathcal{D}_T$, $\mathbf{\mu}_T$, $\mathbf{\Sigma}_T$, $\mathbf{X}_T^T$, $\Phi$
\end{algorithmic}
\end{algorithm}

\subsection{Three-Dimensional Simulation}\label{subsec:simulation2}
This simulation extends the developed framework to a 3D scalar field defined over the cubic domain $\mathcal{X} = [0, 10]^3$. The field is generated using four field sources, each contributing an exponentially decaying intensity centered at fixed ground-plane locations $(x_i, y_i, z_i) \in \{(2,2,0),\; (2,8,0),\; (8,2,0),\; (8,8,0)\}$. The scalar field function is given by
\begingroup\medmuskip=0mu\thickmuskip=0mu\thinmuskip=0mu
\begin{align}
f(x, y, z) = &\scalebox{0.9}{$ A_1 \exp\left( -\lambda \sqrt{(x - x_1)^2 + (y - y_1)^2 + (z - z_1)^2} \right)$} \notag \\
            & \scalebox{0.9}{$+ A_2 \exp\left( -\lambda \sqrt{(x - x_2)^2 + (y - y_2)^2 + (z - z_2)^2} \right)$} \notag \\
            & \scalebox{0.9}{$+ A_3 \exp\left( -\lambda \sqrt{(x - x_3)^2 + (y - y_3)^2 + (z - z_3)^2} \right)$} \notag \\
            & \scalebox{0.9}{$+ A_4 \exp\left( -\lambda \sqrt{(x - x_4)^2 + (y - y_4)^2 + (z - z_4)^2} \right)$},
\label{eq:true_field_3D}
\end{align}
\endgroup
where \( A_i \) is the amplitude of the \( i \)-th source and selected as $A_1,A_4 =40$ and $A_2,A_3 =20$, and \( \lambda = 1.7 \) controls the spatial decay rate of the field.

A binary point cloud constructed using the threshold $\bar{f}= 2$. The high-intensity regions are detected using MATLAB's \texttt{bwlabeln} and \texttt{regionprops3} functions. The resulting high-intensity regions are extracted using the centroids and radii of the detected spheres.

The GP model is initialized with a zero-mean prior, and  \( \alpha =1  \text{ and } l = 2.5 \) are selected as hyperparameters for the kernel. The predicted high-intensity regions are visualized using isosurfaces, and representative 2D slices at \( x = 2 \) and \( x = 8 \) are shown for qualitative comparison with the true field as shown in Figures~\ref{fig:scalar_field_cutouts}-\ref{fig:gp_cutouts}, respectively. The actual and detected high-intensity regions are shown in Figures~\ref{fig:scalar_field_sources1}-\ref{fig:scalar_field_sources2}.

\begin{figure}[htbp]
    \centering
    \includegraphics[width=0.3\textwidth]{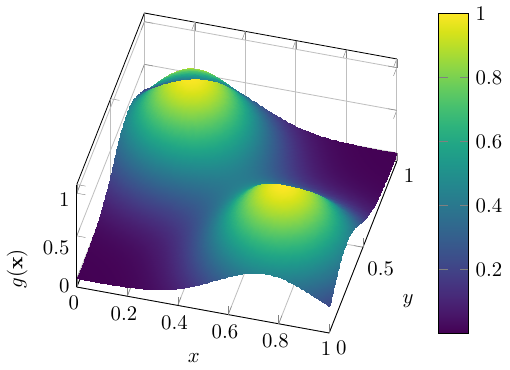}
    \caption{3D Surface plot of the true scalar field.}
    \label{fig:SurfacePlot}
\end{figure}

\begin{figure}[htbp]
    \centering
    \includegraphics[width=0.3\textwidth]{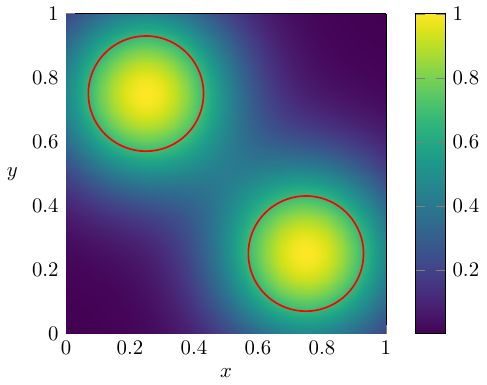}
    \caption{The top view of the scalar field is shown with the high-intensity regions shown as interior regions of circles in red color.}
    \label{fig:SurfacePlot_top_view}
\end{figure}



\begin{figure*}[t]
    \centering
\includegraphics[width=\textwidth]{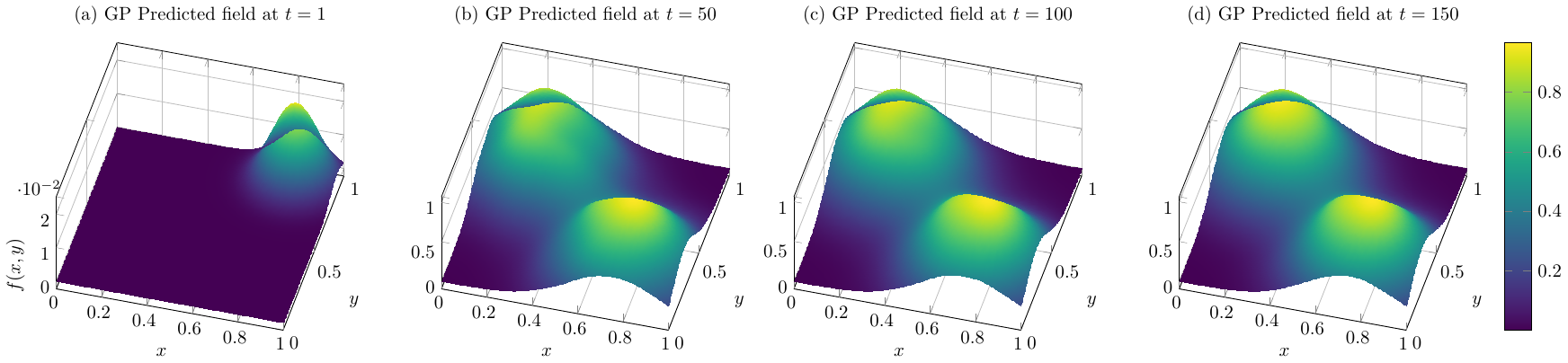}
    \centering
\includegraphics[width=\textwidth]{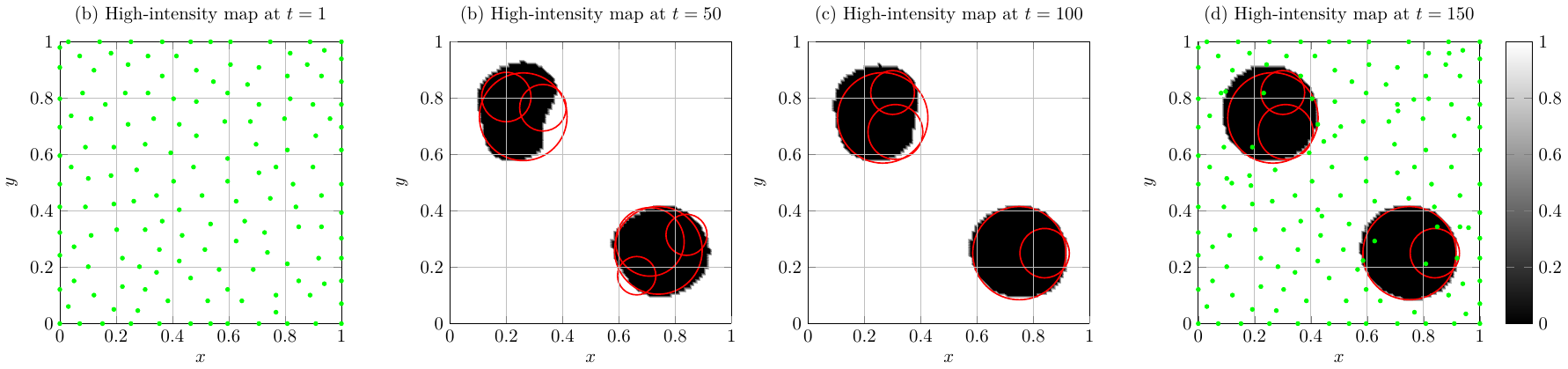}
\caption{Temporal evolution of the GP posterior mean function and estimates of the high-intensity regions. The top row shows the progression of the mean function as the number of measurements increases in plots left to the right at $t=1,50,100$, and $150$, respectively. The bottom row displays the binary image of the high-intensity regions, with the interior of red circles indicating estimated high-intensity regions using the HT. The planned and executed measurement locations are also shown as green dots.}
    \label{fig:temporal_GP-HT-2D}
\end{figure*}

\begin{figure}[htbp]
    \centering
    \includegraphics[width=0.3\textwidth]{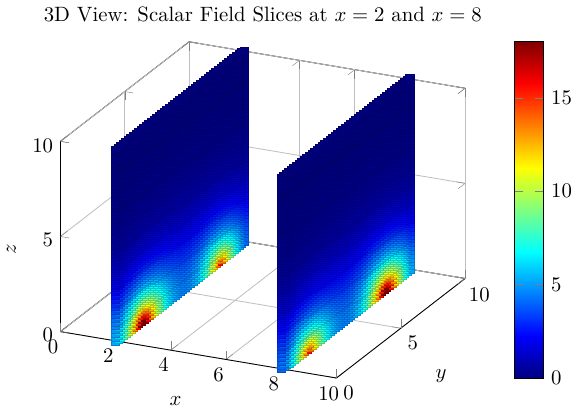}
    \caption{The cutouts of the actual scalar field at $x=2$ and $x=8$.}
    \label{fig:scalar_field_cutouts}
\end{figure}

\begin{figure}[htbp]
    \centering
    \includegraphics[width=0.3\textwidth]{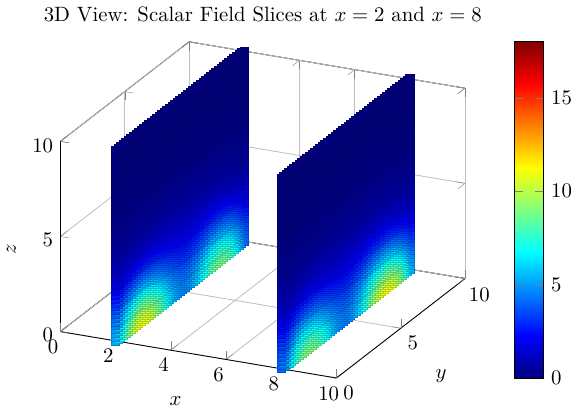}
    \caption{The cutouts of the GP predicted scalar field at $x=2$ and $x=8$.}
    \label{fig:gp_cutouts}
\end{figure}

\begin{figure}[htbp]
    \centering\includegraphics[width=0.3\textwidth]{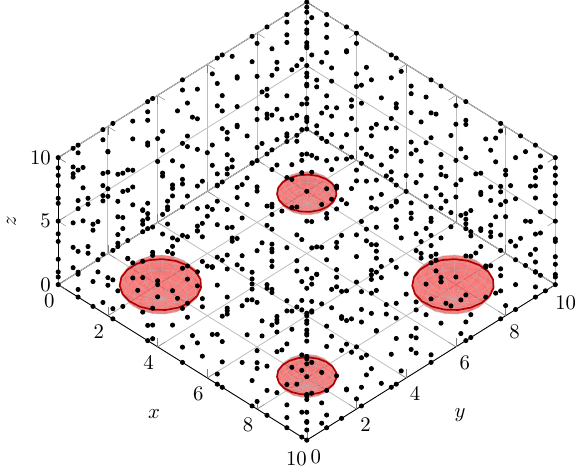}
    \caption{The regions inside the red ellipses denote the regions where the true} field intensity exceeds the given threshold. The black dots show the set of measurement locations assigned to the robot at the commencement of the episode.
    \label{fig:scalar_field_sources1}
\end{figure}
\begin{figure}[htbp]
    \centering
    \includegraphics[width=0.3\textwidth]{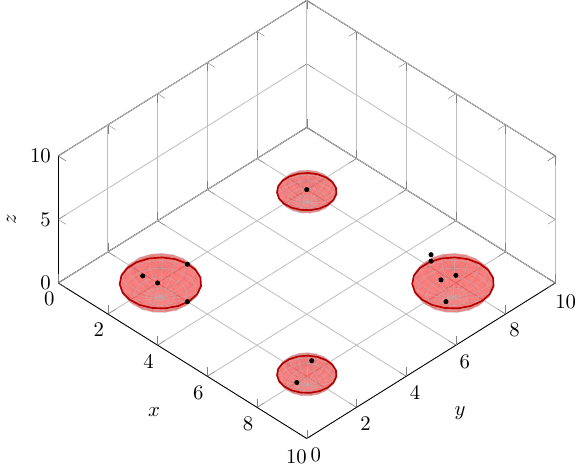}
    \caption{The estimated high-intensity regions in the domain. The planned measurement locations detected in estimated high-intensity regions are also shown as black dots.}
    \label{fig:scalar_field_sources2}
\end{figure}

\subsection{Experimental Results}
\label{subsec:Experimental}
The developed GP-HT algorithm is experimentally validated in an indoor experiment at the University of Florida, where the scalar field is generated using the existing ambient lights in the lab, as shown in Figure \ref{fig:lab_light}. A Husarion 2R-Pro ground robot equipped with an Adafruit TSL2591 lux sensor was used for data collection as shown in Figure \ref{fig:robot_sensor}. The resulting light intensity field is shown in Figure~\ref{fig:ground_lab_light}. A total of 35 measurement locations are planned using the MVS criteria in \eqref{eq:var_criterion}. 
\begin{figure}[htbp]
\centering
\includegraphics[width=0.45\textwidth]{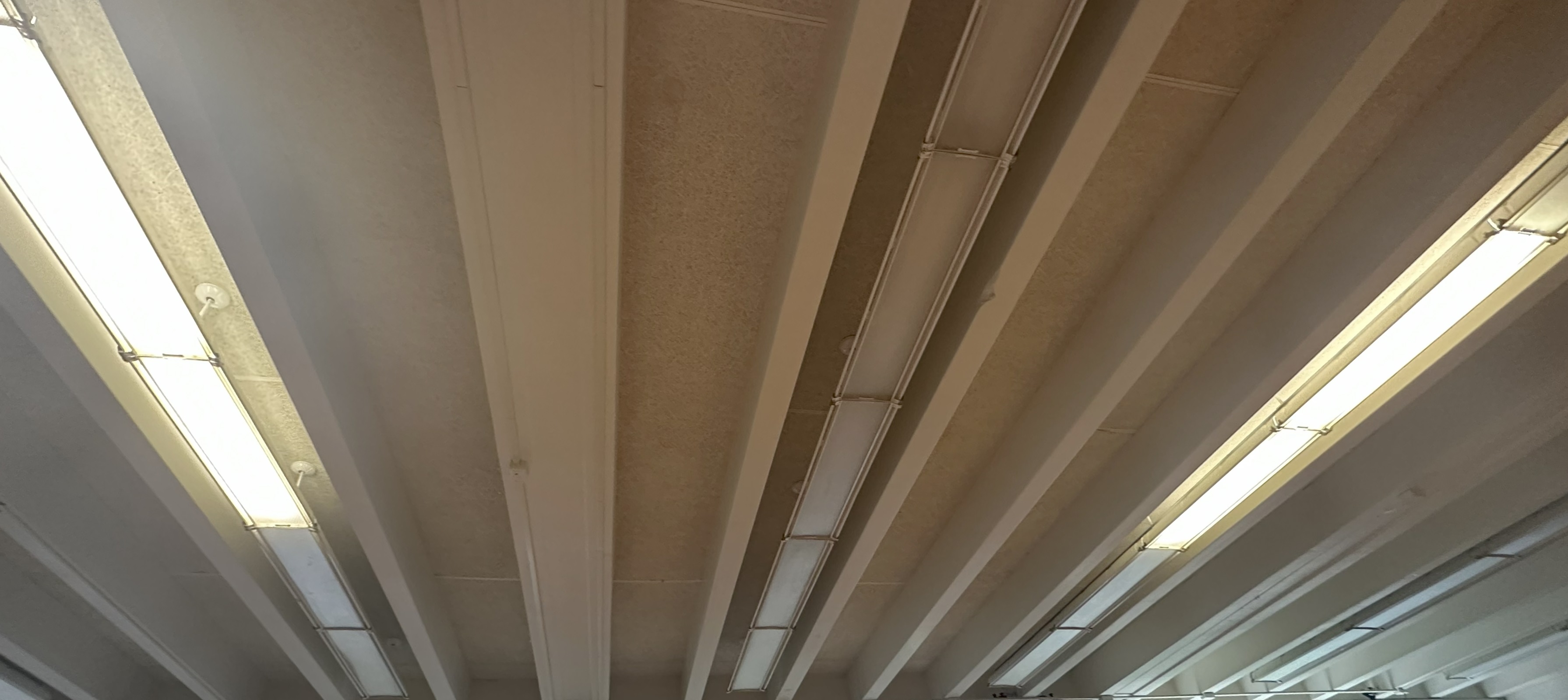}
\caption{Ceiling view of the Lab at the University of Florida, showing the arrangement of overhead lights. The corresponding ground-level light intensity field is shown in Fig.~\ref{fig:ground_lab_light}.}
\label{fig:lab_light}
\end{figure}

\begin{figure}[htbp]
\centering
\includegraphics[width=0.4\textwidth]{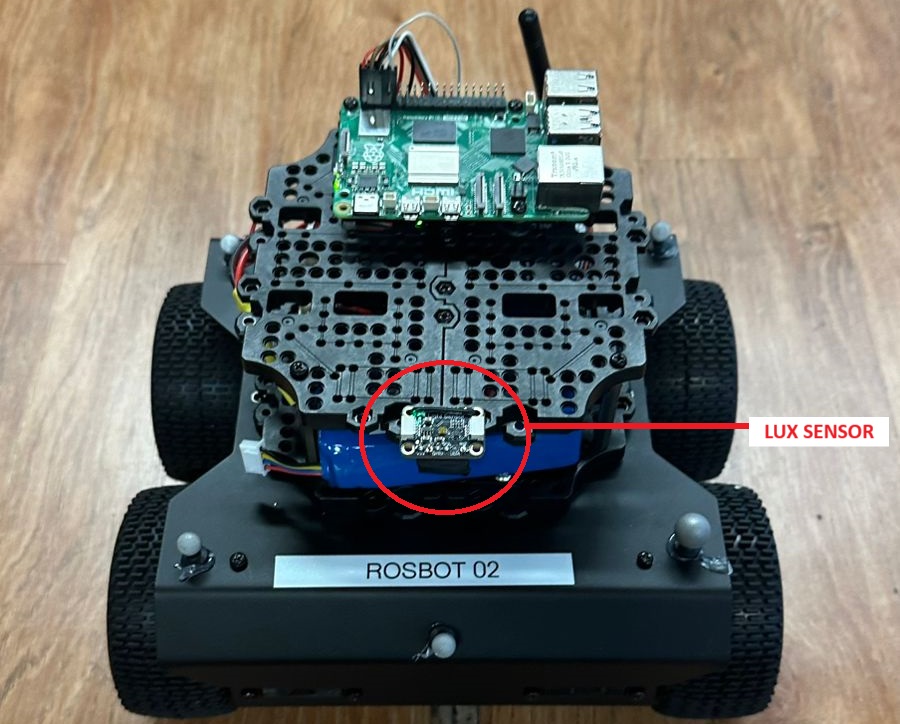}
\caption{Mobile robot used for light intensity measurements in the NCR Lab. The platform is equipped with an Adafruit TSL2591 sensor to record lux readings.}
\label{fig:robot_sensor}
\end{figure}

\begin{figure}[htbp]
\centering
\includegraphics[width=0.5\textwidth]{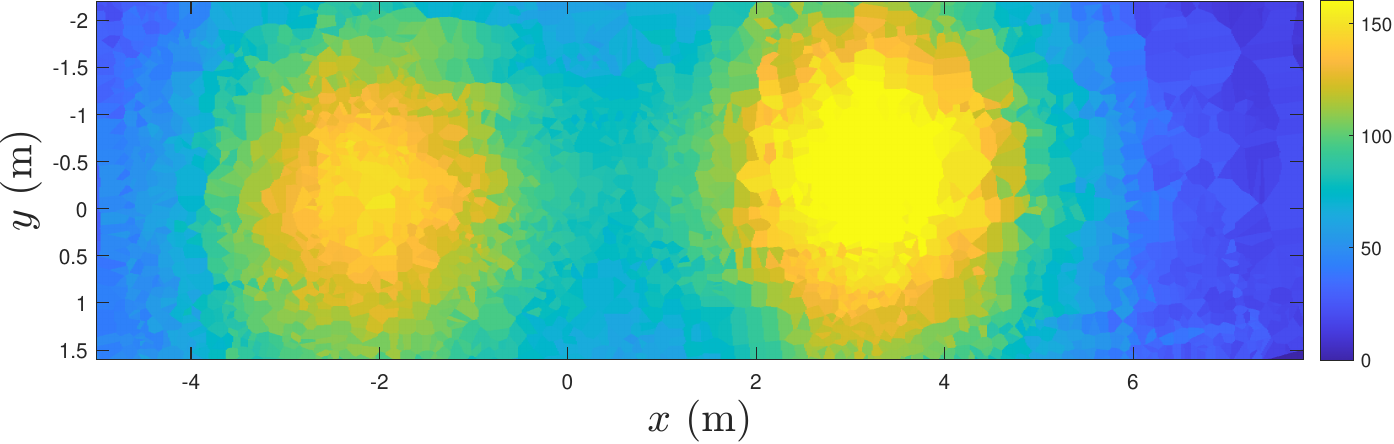}
\caption{Measured ground-level light intensity field produced by the lighting layout in Fig.~\ref{fig:lab_light}. }
    \label{fig:ground_lab_light}
\end{figure}

\begin{figure*}[htbp]
    \centering
    \begin{subfigure}{0.45\textwidth}
        \centering
        \includegraphics[width=\textwidth]{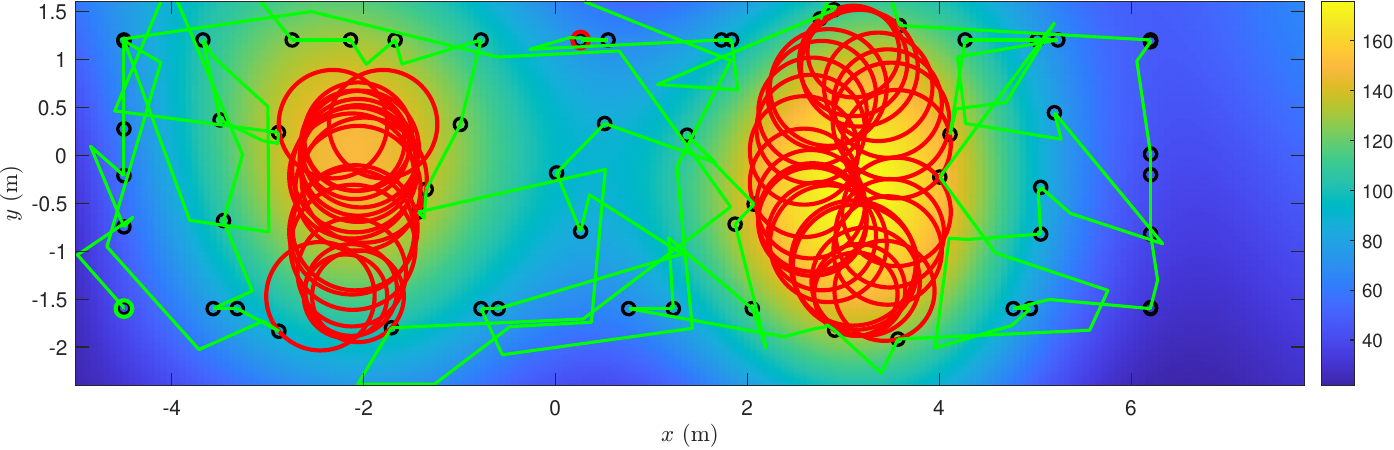}
        \caption{Initial position: bottom-left}
        \label{fig:robot_trajectory_BL}
    \end{subfigure}
    \hfill
    \begin{subfigure}{0.45\textwidth}
        \centering
        \includegraphics[width=\textwidth]{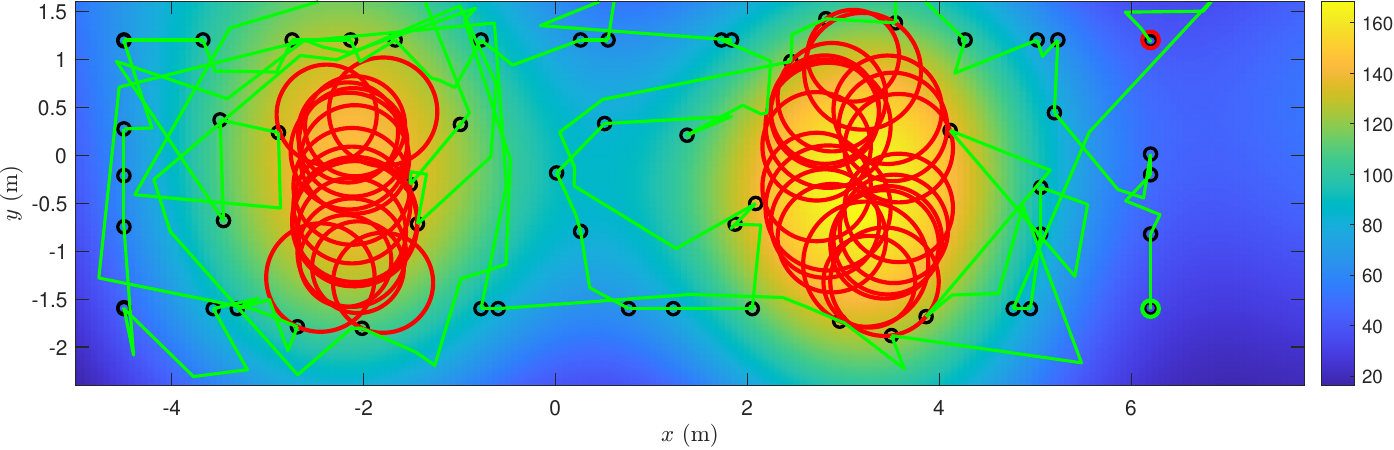}
        \caption{Initial position: bottom-right}
        \label{fig:robot_trajectory_BR}
    \end{subfigure}
    
    \vspace{0.5em}
    \begin{subfigure}{0.45\textwidth}
        \centering
        \includegraphics[width=\textwidth]{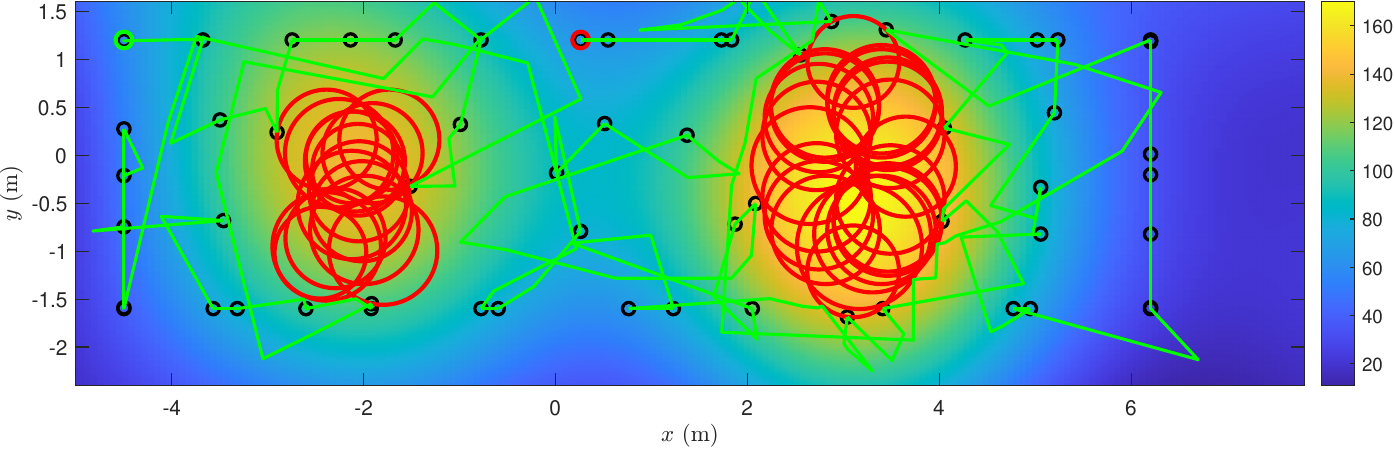}
        \caption{Initial position: top-left}
        \label{fig:robot_trajectory_TL}
    \end{subfigure}
    \hfill
    \begin{subfigure}{0.45\textwidth}
        \centering
        \includegraphics[width=\textwidth]{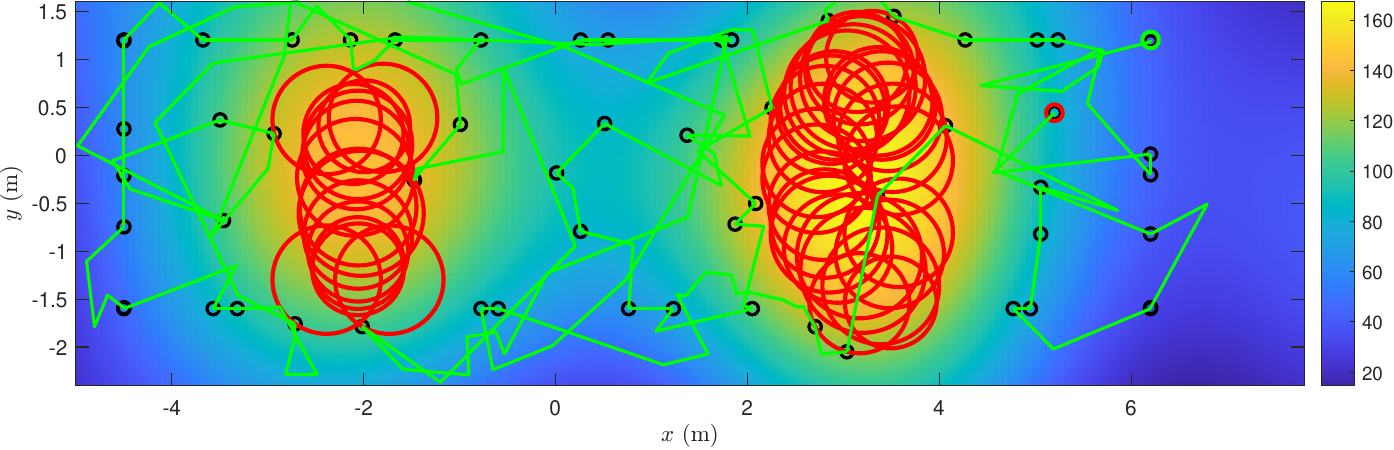}
        \caption{Initial position: top-right}
        \label{fig:robot_trajectory_TR}
    \end{subfigure}
    \caption{Final GP-predicted light intensity fields and corresponding high-intensity regions for four experiments with different initial robot positions. The robot trajectory is shown in green, with start and end positions marked by the green and red circles, respectively.}
    \label{fig:robot_trajectories_all}
\end{figure*}

At each location, light intensity readings are acquired and used to update the GP model. After each update, the HT is applied to the GP posterior to detect high-intensity regions. The RRT* algorithm is then invoked to compute a collision-free path from the current position to the next measurement location. The procedure is repeated when the robot reaches the next measurement location. Four different experiments are performed, starting from different starting positions of the robot as shown in Figures~\ref{fig:robot_trajectories_all}.

\section{Discussion and Conclusion}\label{sec:discussion}
This paper introduced the GP-HT algorithm, an adaptive sampling strategy for scalar field mapping that integrates safety constraints via high-intensity region avoidance.

The simulation and the experimental results validate the performance of the GP-HT algorithm and its ability to map scalar fields while dynamically relocating measurements to avoid high-intensity regions. The observation of executed measurement locations inside the high-intensity regions in Figure \ref{fig:temporal_GP-HT-2D} is attributed to measurements acquired during the early phases of the experiment, where the GP model had not yet acquired sufficient measurements for accurate detection of high-intensity regions. A primary challenge in safe scalar field mapping is the lack of existing techniques that formally integrate safety constraints into Gaussian Process-based methodologies. The GP-HT algorithm addresses this shortcoming by developing a novel strategy to demarcate high-intensity regions online using field measurements, allowing robots to maintain a safe distance without a priori knowledge of hazard locations. A key trade-off is between the need for accurate field estimation and the imperative of safe operation. Due to avoidance constraints, the prediction error converges to a bounded non-zero value, as shown in the convergence analysis.

The performance of the GP-HT algorithm hinges on the selection of the kernel hyperparameters, which are tuned in this work by hand. However, given sufficient data, the algorithm can be implemented using hyperparameter optimization techniques, reducing reliance on hand-tuned parameters. The use of HT to approximate hazardous regions as simple geometric shapes offers computational efficiency but may falter in complex environments with irregular hazard boundaries. Additionally, early-stage vulnerability is evident, as initial measurements may fall within hazardous zones before the GP model sufficiently learns the field structure.

Future work will address several key challenges: enhancing the flexibility of hazard region representations beyond simple geometric shapes. Additionally, extending the framework to dynamic scalar fields where the field source follows a motion plan remains an open challenge.

 \small
\bibliographystyle{IEEETrans.bst}
\bibliography{scc,sccmaster,muzaffar}

\end{document}